\documentclass[11pt]{article}

\usepackage{algorithm, algorithmic}
\usepackage{amsmath,amsfonts,amsthm,amssymb}
\usepackage{color}
\usepackage{natbib}
\usepackage{hyperref}
\usepackage{enumitem}
\usepackage{dsfont}
\usepackage{graphicx}
\usepackage{epstopdf}
\newtheorem{theorem}{Theorem}
\newtheorem{lemma}[theorem]{Lemma}

\newtheorem{definition}[theorem]{Definition}

\usepackage{fullpage}
\newcommand{\abovestrut}[1]{\rule[0in]{0in}{#1}\ignorespaces}
\newcommand{\belowstrut}[1]{\rule[-#1]{0in}{#1}\ignorespaces}
\newcommand{\abovespace}{\abovestrut{0.20in}}

\newcommand{\belowspace}{\belowstrut{0.10in}}
\renewcommand{\cite}{\citep}
\usepackage{subfigure}

\newcommand{\vect}[1]{\ensuremath{\mathbf{#1}}}
\newcommand{\mat}[1]{\ensuremath{\mathbf{#1}}}

\newcommand{\grad}{\nabla}

\newcommand{\argmin}{\mathop{\rm argmin}}
\newcommand{\argmax}{\mathop{\rm argmax}}

\newcommand{\norm}[1]{\|{#1}\|}
\newcommand{\spanvec}[1]{\textrm{span}\left({#1}\right)}

\newcommand{\tr}{\text{tr}}

\newcommand{\trans}{^{\top}}

\newcommand{\st}{\text{s.t.~}}

\newcommand{\R}{\mathbb{R}}

\newcommand{\Wcal}{\mathcal{W}}
\newcommand{\Vcal}{\mathcal{V}}
\newcommand{\geneig}{GenELin}
\newcommand{\geneigk}{GenELinK}

\newcommand{\A}{\mat{A}}
\newcommand{\B}{\mat{B}}
\newcommand{\Binv}{\mat{B}^{-1}}

\renewcommand{\S}{\mat{S}}
\newcommand{\I}{\mat{I}}
\newcommand{\M}{\mat{M}}

\newcommand{\U}{\mat{U}}
\newcommand{\V}{\mat{V}}
\newcommand{\Vc}{\mat{V}_\perp}

\newcommand{\W}{\mat{W}}
\newcommand{\X}{\mat{X}}
\newcommand{\Y}{\mat{Y}}
\newcommand{\vb}{\vect{b}}
\renewcommand{\u}{\vect{u}}

\newcommand{\Tcal}[1]{\mathcal{T}\left(#1\right)}
\renewcommand{\v}{\vect{v}}
\newcommand{\m}{\vect{m}}
\newcommand{\w}{\vect{w}}
\newcommand{\x}{\vect{x}}
\newcommand{\y}{\vect{y}}

\newcommand{\bphi}{\vect{\phi}}
\newcommand{\bpsi}{\vect{\psi}}

\newcommand{\cn}{\kappa}
\newcommand{\nn}{\nonumber}
\newcommand{\defeq}{\stackrel{\mathrm{def}}{=}}
\newcommand{\order}[1]{O\left(#1\right)}
\newcommand{\otilde}[1]{\widetilde{O}\left(#1\right)}
\newcommand{\nnz}[1]{\mathrm{nnz}\left(#1\right)}
\newcommand{\abs}[1]{\left|#1\right|}
\newcommand{\iprod}[2]{\langle #1,#2\rangle}
\newcommand{\singmin}[1]{\sigma_{\textrm{min}}\left(#1\right)}
\newcommand{\singmax}[1]{\sigma_{\textrm{max}}\left(#1\right)}

 \newcommand{\eye}{\mat{I}}

\title{Efficient Algorithms for Large-scale Generalized Eigenvector Computation and Canonical Correlation Analysis}

\author{Rong Ge\footnote{Duke University. Email: rongge@cs.duke.edu} \and Chi Jin\footnote{UC Berkeley. Email: chijin@cs.berkeley.edu} \and Sham M. Kakade\footnote{University of Washington. Email: sham@cs.washington.edu} \and Praneeth Netrapalli\footnote{Microsoft Research New England. Email: praneeth@microsoft.com} \and Aaron Sidford\footnote{Microsoft Research New England. Email: asid@microsoft.com}}

\begin{document}
\maketitle

\begin{abstract}
This paper considers the problem of canonical-correlation
analysis (CCA)~\cite{Hotelling1936} and, more broadly, the generalized
eigenvector problem for a pair of symmetric matrices. These are two
fundamental problems in data analysis and scientific computing with
numerous applications in machine learning and statistics
\cite{ShiM2000,HardoonSS2004,WittenTH2009}.  

We provide simple iterative algorithms, with improved runtimes, for solving these problems that
are globally linearly convergent with moderate dependencies 
on the
condition numbers and eigenvalue gaps of the matrices involved. 

We obtain our results by reducing CCA to the top-$k$ generalized
eigenvector problem. 
We solve this problem through a general framework that simply requires black box access to an approximate linear system solver. 
Instantiating this framework with accelerated gradient descent we
obtain a running time of $\order{\frac{z k \sqrt{\kappa}}{\rho}
  \log(1/\epsilon) \log \left(k\kappa/\rho\right)}$ where $z$ is the
total number of nonzero entries, $\kappa$ is the
condition number 
and 
$\rho$ is the relative eigenvalue gap of the appropriate matrices. 

Our algorithm is linear in the input size and the number of components
$k$ up to a $\log(k)$ factor. This is essential for handling large-scale matrices that appear
in practice. To the best of our knowledge this is the first such
algorithm with global linear convergence. 
We hope that our results prompt further research 
and ultimately improve the practical running time for performing these important data analysis procedures on large data sets. 
\end{abstract}


\section{Introduction}

Canonical-correlation analysis (CCA) and the generalized eigenvector problem
are fundamental problems in scientific computing, data analysis, and statistics \cite{barnett1987origins,friman2001detection}.

These problems arise naturally in statistical settings.
Let  $\X, \Y \in \R^{n \times d}$ denote two large sets of data points, with empirical covariance matrices $\S_x = \frac{1}{n} \X^\top \X$, $\S_y = \frac{1}{n} \Y^\top \Y$, and $\S_{xy} = \frac{1}{n} \X^\top \Y$ and suppose we wish to find features $\x,\y \in \R^d$ that best encapsulate the similarity or dissimilarity of the data sets. CCA is the problem of maximizing the empirical correlation
\begin{equation}
\label{eq:intro-cca}
\max_{\x^\top \S_{xx} \x = 1 \text{ and } \y^\top \S_{yy} \y = 1}
\x^\top \S_{xy} \y
\end{equation}
and thereby extracts common features of the data sets. On the other hand the generalized eigenvalue problems
\[
\max_{\x \neq 0} \frac{\x^\top \S_{xx} \x}{\x^\top \S_{yy} \x}
\enspace \text{ and } \enspace
\max_{\y \neq 0} \frac{\y^\top \S_{yy} \y}{\y^\top \S_{xx} \y}
\]
compute features that maximizes discrepancies between the data sets. Both these problems are easily extended to the $k$-feature case (See Section~\ref{sec:prob}). 
Algorithms for solving them are commonly used to extract features to compare and contrast large data sets and are used commonly in regression \cite{KakadeF2007}, clustering \cite{chaudhuri2009multi}, classification \cite{KarampatziakisM2013}, word embeddings \cite{dhillon2011multi} and more.

Despite the prevalence of these problems and the breadth of research on solving them in practice (\cite{barnett1987origins,barnston1992prediction,sherry2005conducting,KarampatziakisM2013} to name a few), there are relatively few results on obtaining provably efficient algorithms. Both problems can be reduced to performing principle component analysis (PCA), albeit on complicated matrices e.g $\S_{yy}^{-1/2} \S_{xy}^\top \S_{xx}^{-1} \S_{xy} \S_{yy}^{-1/2}$ for CCA and $\S_{yy}^{-1/2} \S_{xx} \S_{yy}^{-1/2}$ for generalized eigenvector.  However applying PCA to these matrices traditionally involves the formation of  $\S_{xx}^{-1/2}$ and $\S_{yy}^{-1/2}$ which is prohibitive for sufficiently large datasets if we only want to estimate top-$k$ eigenspace.

A natural open question in this area is to what degree can the formation of $\S_{xx}^{-1/2}$ and $\S_{yy}^{-1/2}$ can be bypassed to obtain efficient scalable algorithms in the case where the number of features $k$ is much smaller than the dimensions of the problem $n$ and $d$. Can we develop simple iterative practical methods that solve this problem in close to linear time when $k$ is small and the condition number and eigenvalue gaps are bounded?
While there has been recent work on solving these problems using iterative methods \cite{avron2014efficient,paul2015core,lu2014large,deanCCA} we are unaware of previous provable global convergence results and more strongly, linearly convergent scalable algorithms.

The central goal of this paper is to answer this question in the affirmative. We present simple globally linearly convergent iterative methods that solve these problems. The running time of these problems scale well as the number of features and conditioning of the problem stay fixed and the size of the datasets grow.  Moreover, we implement the method and perform experiments demonstrating that the techniques may be effective for large scale problems. 

Specializing our results to the single feature case we show how to solve the problems all in time $O(\frac{z \sqrt{\kappa}}{\rho}\log \frac{1}{\rho} \log \frac{1}{\epsilon})$, where $\kappa$ is the maximum of condition numbers of $\S_{xx}$ and $\S_{yy}$ and $\rho$ is the eigengap of appropriate matrices and mentioned above, and $z$ is the number of nonzero entries in $\X$ and $\Y$. 
 To the best of our knowledge this is the first such globally linear convergent algorithm for solving these problems. 

We achieve our results through a general and versatile framework that allows us to utilize fast linear system solvers in various regimes. We hope that by initiating this theoretical and practical analysis of CCA and the generalized eigenvector problem we can promote further research on the problem 
and ultimately advance the state-of-the-art for efficient data analysis.


\subsection{Our Approach}

To solve the problems motivated in the previous section we first directly reduce CCA to a generalized eigenvector problem (See Section~\ref{sec:cca}). Consequently, for the majority of the paper we focus on the following:

\begin{definition}[Top-$k$ Generalized Eigenvector\footnote{We use the term \emph{generalized eigenvector} to refer to 
a non-zero vector $\v$ such that $\A \v = \lambda \B \v$ for symmetric $\A$ and $\B$, not the general notion of eigenvectors for asymmetric matrices.}]
\label{def:topk_evec_intro}
Given symmetric matrices $\A, \B $ where  $\B$ is positive definite compute 
$\w_1,\cdots,\w_k$ defined for all $i\in[k]$ by 
\begin{align*}
&\w_i \in \argmax_{\w} \abs{\w \trans \A \w}
~
\st
\begin{array}{c}
\w \trans \B \w = 1 \mbox{ and } \\
\w \trans \B \w_j = 0 \; \forall \; j \in [i-1].
\end{array} 
\end{align*}
\end{definition}

The generalized eigenvector is equivalent to the problem of computing the PCA of $\A$ in the $\B$ norm. Consequently, it is the same as computing the top $k$ eigenvectors of largest absolute value of the symmetric matrix $\M = \B^{-1/2} \A \B^{-1/2}$ and then multiplying by $\B^{-1/2}$. 

Unfortunately, as we have discussed, explicitly computing $\B^{-1/2}$ is prohibitively expensive when $n$ is large and therefore we wish to avoid forming $\M$ explicitly. One natural approach is to develop an iterative methods to approximately apply $\B^{-1/2}$ to a vector and then use that method as a subroutine to perform the power method on $\M$. Even if we could perform the error analysis to make this work, such an approach would likely require at least a suboptimal $\Omega(\log^2(1/\epsilon))$ iterations to achieve error $\epsilon$. 

To bypass these difficulties, we take a closer look at the power method. For some initial vector $\x$, let $\y = \B^{-1/2} \M^{i} \x$ be the result of $i$ iterations of power method on $\M$ followed by multiplying $\B^{-1/2}$. Clearly $\y = (\B^{-1} \A)^i \B^{-1/2} \x$. Furthermore, since we typically initialize the power method by a random vector and since $\B$ is positive definite, if we instead we computed $\y = (\B^{-1} \A)^i \x$ for random $\x$ we would likely converge at the same rate as the power method at the cost of just a slightly worse initialization quality.

Consequently, we can compute our desired eigenvectors by simply alternating between applying $\A$ and $\B^{-1}$ to a random initial vector. Unfortunately, computing $\B^{-1}$ exactly is again outside our computational budget. At best we should only attempt to apply $\B^{-1}$ approximately by linear system solvers. 

One of our main technical contributions is to argue about the effect of inexact solvers in this method. 
Whereas solving every linear system to target accuracy $\epsilon$ would again require $O(\log(1/\epsilon))$ time per linear system, which leads to a sub-optimal $O(\log^2(1/\epsilon))$ overall running time, i.e. sublinear convergence, we instead show how to warm start the linear system solvers and obtain a faster rate. 
We exploit the fact that as we perform many iterations of power methods, points at time $t$ converge to eigenvectors and therefore we can initialize our linear system solver at time $t$ carefully using our points at time $t-1$. Ultimately we show that we only need to make fixed multiplicative progress in solving the linear system in every iteration of the power method, thus the runtime for solving each linear system is independent of $\epsilon$.

Putting these pieces together with careful error analysis yields our main result. Our algorithm only requires the ability to apply $\A$ to a vector and an approximate linear system solver for $\B$, which in turn can be obtained by just applying $\B$ to vectors. Consequently, our framework is versatile, scalable, and easily adaptable to take advantage of faster linear system solvers. 





\subsection{Previous Work}

While there has been limited previous work on provably solving CCA and generalized eigenvectors, we note that there is an impressive body of literature on performing PCA\cite{rokhlin2009randomized,halko2011finding,musco2015stronger,garber2015fast,jin2015fast} and solving positive semidefinite linear systems\cite{hestenes1952methods,nesterov1983method,spielman2004nearly}. Our analysis in this paper draws on this work extensively and our results should be viewed as the principled application of them to the generalized eigenvector problem. 

There has been much recent interest in designing scalable algorithms
for CCA\cite{deanCCA, karenCCA1,karenCCA2,karenCCA3}.  To our
knowledge, there are no provable guarantees for approximate methods for
this problem.
Heuristic-based approachs \cite{WittenTH2009, lu2014large} compute efficiently, 
but only give suboptimal result due to coarse approximation.
The work in \cite{deanCCA} provides one natural iterative procedure,
where the per iterate computational complexity is low. This work only
provides local convergence guarantees and does not provide guarantees
of global convergence. 

Also of note is that many recent algorithms~\cite{deanCCA, karenCCA1}
have mini-batch variations, but there's no guarantees for mini-batch style algorithm for CCA yet. Our algorithm can also be easily extends to a mini-batch version. While we do not explicitly analyze this variation, and we believe our analysis and techniques are helpful for extensions to this setting. 
We also view this as an important direction for future work.


We hope that by establishing the generalized eigenvector problem and providing provable guarantees under moderate regularity assumptions that our results may be further improved and ultimately this may advance the state-of-the-art in practical algorithms for performing data analysis. 

\subsection{Our Results}
Our main result in this paper is a linearly convergent algorithm for computing the top generalized eigenvectors (see Definition~\ref{def:topk_evec_intro}). In order to be able to state our results we introduce some notation. Let $\lambda_1,\cdots,\lambda_d$ be the eigenvalues of $\Binv \A$ (their existence is guaranteed by Lemma~\ref{lem:eig-sing} in the appendix). The eigengap $\rho \defeq 1 - \frac{\abs{\lambda_{k+1}}}{\abs{\lambda_{k}}}$ and $\gamma \defeq \frac{\abs{\lambda_1}}{\abs{\lambda_k}}$. Let $z$ denote the number of nonzero entries in $\A$ and $\B$.
\begin{theorem}[Informal version of Theorem~\ref{thm:main_k}]
	Given two matrices $\A$ and $\B \in \R^{d\times d}$, there is an algorithm that computes the top-$k$ generalized eigenvectors up to an error $\epsilon$ in time $\widetilde{O}(\frac{z k \sqrt{\cn\left(\B\right)}}{\rho} \log \frac{1}{\epsilon})$,
	where $\cn\left(\B\right)$ is the condition number of $\B$ and $\otilde{\cdot}$ hides logarithmic terms in $d$, $\gamma$, $\cn\left(\B\right)$ and $\rho$, and nothing else.
\end{theorem}
Here is a comparison of our result with previous work.
\begin{table}[h!]
	\caption{Runtime Comparison - Generalized Eigenvectors}
	\begin{center}
		\begin{small}
			\begin{sc}
			{\renewcommand{\arraystretch}{1.5}
				\begin{tabular}{|c|c|}
					\hline
					\abovespace\belowspace
					\geneigk (this paper) &	$\widetilde{O}(\frac{d^2 k \sqrt{\cn\left(\B\right)}}{\rho} \log \frac{1}{\epsilon})$ \\
					\hline
					Fast matrix inversion	& $O(d^{2.373...})$\\
					\hline
				\end{tabular}
				}
			\end{sc}
		\end{small}
	\end{center}
\end{table}

Turning to the problem of CCA, cf.~\eqref{eq:intro-cca}, the relevant parameters are $\kappa \defeq \max\left(\cn\left(\S_{xx}\right),\cn\left(\S_{yy}\right)\right)$ i.e., the maximum of the condition numbers of $\S_{xx}$ and $\S_{yy}$, $\gamma \defeq \frac{\abs{\lambda_1}}{\abs{\lambda_k}}$ where $\lambda_1,\cdots,\lambda_k$ are the eigenvalues of $\S_{yy}^{-1}\S_{yx}\S_{xx}^{-1}\S_{xy}$ in decreasing absolute value. Let $z$ denote the number of nonzeros in $\X$ and $\Y$.
Our main results are a reduction from CCA to the generalized eigenvector problem.
\begin{theorem}[Informal version of Theorem~\ref{thm:main_cca}]\label{thm:main_cca_informal}
	Given two data matrices $\X \in \R^{d_1 \times n}$ and $\Y \in \R^{d_2 \times n}$, there is an algorithm that performs top-$k$ CCA up to an error $\epsilon$ in time $\widetilde{O}(\frac{zk \sqrt{\cn}}{\rho} \log \frac{1}{\epsilon})$,
	where $d = d_1+d_2$ and $\otilde{\cdot}$ hides logarithmic terms in $d$, $\gamma$, $\cn$ and $\rho$, and nothing else.
\end{theorem}
Table~\ref{tab:CCA} compares our result with existing results. \footnote{\cite{deanCCA} only shows local convergence for S-AppGrad. Starting within this radius of convergence requires us to already solve the problem to a high accuracy.}
\begin{table}[h!]
	\caption{Runtime Comparison - CCA}
	\label{tab:CCA}
	\begin{center}
		\begin{small}
			\begin{sc}
			{\renewcommand{\arraystretch}{1.5}
				\begin{tabular}{|c|c|}
					\hline
					\abovespace\belowspace
					This paper &	$\widetilde{O}(\frac{nd k \sqrt{\cn}}{\rho} \log \frac{1}{\epsilon})$ \\
					\hline
					S-AppGrad~\cite{deanCCA} & $\widetilde{O}(\frac{nd k \cn }{\rho^2}\log \frac{1}{\epsilon})$\\
					\hline
					Fast matrix inversion	& $O(n d^{1.373...})$\\
					\hline
				\end{tabular}
				}
			\end{sc}
		\end{small}
	\end{center}
\end{table}

We should note that the actual bounds we obtain are somewhat stronger than the above informal bounds. Some of the terms in logarithm also appear only as additive terms. Finally we also give natural stochastic extensions of our algorithms where the cost of each iteration may be much smaller than the input size. The key idea behind our approach is to use an approximate linear system solver as a black box inside power method on an appropriate matrix. We show that this dependence on a linear system solver is in some sense essential. In Section~\ref{sec:reduction} we show that the generalized eigenvector problem is strictly more general than the problem of solving positive semidefinite linear systems and consequently our dependence on the condition number of $\B$ is in some cases optimal.

\begin{table}[h!]
	\caption{Runtime Comparison - CCA with Different Linear Solver\protect\footnotemark}
	\label{tab:CCA_solver}
	\begin{center}
		\begin{small}
			\begin{sc}
			{\renewcommand{\arraystretch}{1.5}
				\begin{tabular}{|c|c|}
					\hline
					\abovespace\belowspace
					Our result + GD &	$\widetilde{O}(\frac{nd k \cn}{\rho} \log \frac{1}{\epsilon})$ \\
					\hline
					Our result + AGD &	$\widetilde{O}(\frac{nd k \sqrt{\cn}}{\rho} \log \frac{1}{\epsilon})$ \\
					\hline
					Our result + SVRG &	$\widetilde{O}(\frac{d k(n+\tilde{\cn})}{\rho} \log \frac{1}{\epsilon})$ \\
					\hline
					Our result + ASVRG & $\widetilde{O}(\frac{d k (n+\sqrt{n\tilde{\cn}})}{\rho} \log \frac{1}{\epsilon})$ \\
					\hline
				\end{tabular}
				}
			\end{sc}
		\end{small}
	\end{center}
\end{table}
\footnotetext{This table was inspired by \cite{wang2016globally} in order to facilitate comparison to existing work.}

Subsequent to the submission of this paper,  we learned of the
closely related work in \cite{wang2016globally}, which presents a number of
additional interesting results.
We think it is worthwhile to point out that our algorithm only requires black box access to any linear solver. Although the result in Theorem \ref{thm:main_cca_informal} was stated by instantiating the linear system solver by accelerated gradient descent (AGD), it is immediate to apply Theorem \ref{thm:main_cca} and give the corresponding rates if we instantiate it by other popular algorithms, including gradient descent (GD), stochastic variance reduction (SVRG) \cite{johnson2013accelerating}, and its accelerated version (ASVRG) \cite{frostig2015regularizing,lin2015catalyst}.
We summarize the corresponding runtime in Table \ref{tab:CCA_solver}.
There $\tilde{\cn} \defeq \max\left(\frac{\max_i \norm{\x_i}^2}{\sigma_{\min}(\S_{xx})}
,\frac{\max_i \norm{\y_i}^2}{\sigma_{\min}(\S_{yy})}\right) $, and $\x_i$, $\y_i$ are $i$-th column of matrix $\X$ and $\Y$.  Note by definition we always have $\tilde{\cn} \ge \cn$.
For generalized eigenvector problem, results of similar flavor as in Table \ref{tab:CCA_solver} can also be easily derived.

Finally, we also run experiments to demonstrate the practical effectiveness of our algorithm on both small and large scale datasets.

\subsection{Paper Overview}

In Section~\ref{sec:notation}, we present our notation. In Section~\ref{sec:prob}, we formally define the problems we solve and their relevant parameters. In Section~\ref{sec:our_results}, we present our results for the generalized eigenvector problem. In Section~\ref{sec:cca}, we present our results for the CCA problem. In Section~\ref{sec:reduction} we argue that generalized eigenvector computation is as hard as linear system solving and that our dependence on $\cn(\B)$ is near optimal. In Section~\ref{sec:exps}, we present experimental results of our algorithms on some real world data sets. Due to space limitations, proofs are deferred to the appendix.

\section{Notation}\label{sec:notation}
We use bold capital letters $(\A,\B,\cdots)$ to denote matrices and bold lowercase letters $(\u,\v,\cdots)$ for vectors. For symmetric positive semidefinite (PSD) matrix $\B$, we let $\norm{\u}_{\B}  \defeq \sqrt{\u \trans \B \u}$ denote the $\B$-norm of $u$ and we let $\iprod{\u}{\v}_{\B}  \defeq \u \trans \B \v$ denotes the inner product of $\u$ and $\v$ in the $\B$-norm. We say that a matrix $\W$ is $\B$-orthonormal if $\W\trans \B \W=\eye$. We let $\sigma_i(\A)$ denotes the $i^{\textrm{th}}$ largest singular value of $\A$, $\singmin{\A}$ and $\singmax{\A}$ denote the smallest and largest singular values of $\A$ respectively. Similarly we let $\lambda_i(\A)$ refers to the $i^{\textrm{th}}$ largest eigenvalue of $\A$ in magnitude. We let $\nnz{\A}$ denotes the number of nonzeros in $\A$. We also let $\cn(\B)$ denote the condition number of $\B$ (i.e., the ratio of the largest to smallest eigenvalue).


\section{Problem Statement}\label{sec:prob}
In this section, we recall the generalized eigenvalue problem, define our error metric, and introduce all relevant parameters.
Recall that the generalized eigenvalue problem is to find $k$ vectors $\w_i$, $i\in [k]$ such that
\begin{align*}
	&\w_i \in \argmax_{\w} \abs{\w \trans \A \w} \quad
	\st
	\begin{array}{c}
		\w \trans \B \w = 1 \mbox{ and } \\
		\w \trans \B \w_j = 0 \; \forall \; j \in [i-1].
	\end{array}
\end{align*}
Using stationarity conditions, it can be shown that the vectors $\w_i$ are given by $\w_i=\v_i$, where $v_i$ is an eigenvector of $\Binv \A$ with eigenvalue $\lambda_i$ such that $\abs{\lambda_1} \geq \cdots \geq \lambda_n$.
Our goal is to recover the top-k eigen space i.e., $\textrm{span}\{\v_1, \cdots, \v_k\}$. In order to quantify the error in estimating the eigenspace, we use largest principal angle, which is a standard notion of distance between subspaces~\cite{golub2012matrix}.
\begin{definition}[Largest principal angle]\label{def:lp_angle}
	Let $\Wcal$ and $\Vcal$ be two $k$ dimensional subspaces, $\W$ and $\V$ their $\B$-orthonormal basis respectively. The largest principal angle $\theta\left(\Wcal,\Vcal\right)$ in the $\B$-norm is defined to be
	\begin{align*}
	\theta\left(\Wcal,\Vcal\right) \defeq \arccos\left( \singmin{\V\trans \B \W}\right),
	\end{align*}
\end{definition}
Intuitively, the largest principal angle corresponds to the largest angle between any vector in the span of $\W$ and its projection onto the span of $\V$. In the special case where $k=1$, the above definition reduces to our choice in the top-$1$ setting. Given two matrices $\W$ and $\V$, we use $\theta\left(\W,\V\right)$ to denote the largest principle angle between the subspaces spanned by the columns of $\W$ and $\V$. We say that $\W$ achieves an error of $\epsilon$ if $\W\trans \B\W = \eye$ and $\sin \theta\left(\W,\V\right) \leq \epsilon$, where $\V$ is the $d\times k$ matrix whose columns are $\v_1,\cdots,\v_k$. The relevant parameters for us are the eigengap, i.e. the relative difference between $k^{\textrm{th}}$ and $(k+1)^{\textrm{th}}$ eigenvalues, $\rho \defeq 1-\frac{\abs{\lambda_{k+1}}}{\abs{\lambda_k}}$, and $\cn(\B)$, the condition number of $\B$.

\section{Our Results}
\label{sec:our_results}

In this section, we provide our algorithms and results for solving the generalized eigenvector problem. We present our results for the special case of computing the top generalized eigenvector (Section~\ref{sub:results-top-1}) followed by the general case of computing the top-$k$ generalized eigenvectors (Section~\ref{sec:algo-k}). However, first we formally define a linear system solver as follows:

\textbf{Linear system solver}:
In each of our main results (Theorems~\ref{thm:main} and~\ref{thm:main_k}) we assume black box access to an approximate linear system solver. Given a PSD matrix $\B$, a vector $\vb$, an initial estimate $\u_0$, and an error parameter $\delta$, we require to decrease the error by a multiplicative $\delta$, i.e. output $\u_1$ with $\norm{\u_1 - \Binv \vb}_{\B}^2 \leq \delta \norm{\u_0 - \Binv \vb}_{\B}^2$. We let $\Tcal{\delta}$ denote the time needed for this operation. Since the error metric $\norm{\u_1 - \Binv\vb}_{\B}^2$ is equivalent to function error on minimizing the convex quadratic $f(\u) \defeq \frac{1}{2} \u \trans \B \u - \u \trans \vb$ up to constant scaling, an approximate linear system solver is equivalent to an optimization algorithm for $f(\u)$. We also specialize our results using Nesterov's accelerated gradient descent to state our bounds. Stating our results using linear system solver as a blackbox allows the user to choose an efficient solver depending on the structure of $\B$ and helps pass any improvements in linear system solvers on to the problem of generalized eigenvectors.

\subsection{Top-1 Setting}
\label{sub:results-top-1}
Our algorithm for computing the top generalized eigenvector, \geneig~is given in Algorithm~\ref{algo:appro-power}.
\begin{algorithm}[h!]
	\caption{\textbf{Gen}eralized \textbf{E}igenvector via \textbf{Lin}ear System Solver (GenELin)} \label{algo:appro-power}
	\begin{algorithmic}
		\renewcommand{\algorithmicrequire}{\textbf{Input: }}
		\renewcommand{\algorithmicensure}{\textbf{Output: }}
		\REQUIRE $T$, symmetric matrix $\A$, PSD matrix $\B$.
		\ENSURE  top generalized eigenvector $\w$.
		\STATE $\tilde{\w}_0 \leftarrow$ sample uniformly from unit sphere in $\R^d$
		\STATE $\w_0 \leftarrow {\tilde{\w}}_0/\norm{\tilde{\w}_0}_\B$
		\FOR{$t = 0,\cdots,T-1$}
		\STATE $\beta_t \leftarrow  \w_t\trans \A \w_t  / \w_t\trans \B \w_t$
		\STATE $\tilde{\w}_{t+1} \leftarrow \argmin_{\w \in \R^d} [\frac{1}{2}\w\trans\B\w - \w\trans\A\w_t]$\\ \quad\quad\quad\{Use an optimization subroutine \\
		 \quad\quad\quad~with initialization $\beta_t \w_t$ \}
		\STATE $\w_{t+1} \leftarrow \tilde{\w}_{t+1}/\norm{\tilde{\w}_{t+1}}_\B$
		\ENDFOR
		\STATE \textbf{Return} $\w_T$.
	\end{algorithmic}
\end{algorithm}

The algorithm implements an approximate power method where each iteration consists of approximately multiplying a vector by $\Binv \A$. In order to do this, \geneig~solves a linear system in $\B$ and then scales the resulting vector to have unit $\B$-norm. Our main result states that given an oracle for solving the linear systems,\footnote{For example, we could use Nesterov's accelerated gradient descent, Algorithm~\ref{algo:AGD}} the number of iterations taken by Algorithm~\ref{algo:appro-power} to compute the top eigenvector up to an accuracy of $\epsilon$ is at most $\frac{4}{\rho} \log \frac{1}{\epsilon \cos \theta_0 }$ where $\theta_0 \defeq \theta\left(\w_0,\v_1\right)$.
\begin{theorem} \label{thm:main}
	Recall that the linear system solver takes time $\Tcal{\delta}$ to reduce the error by a factor $\delta$. Given matrices $\A$ and $\B$, 
	\geneig~(Algorithm~\ref{algo:appro-power}) computes a vector $\w_T$ achieving an error of $\epsilon$	in $T= \frac{2}{\rho} \log \frac{1}{\epsilon \cos \theta_0}$ iterations, where $\theta_0 \defeq \theta\left(\w_0,\v_1\right)$.
	The running time of the algorithm is at most 
	\begin{align*}
	&O\left(\frac{1}{\rho} \left( \log\frac{1}{\cos \theta_0} \cdot  \Tcal{\frac{\rho^2 \cos^2 \theta_0 }{16}}+\log \frac{1}{\epsilon} \cdot \Tcal{\frac{\rho^2}{16}} \right)+ \frac{1}{\rho} (\nnz{\A}+\nnz{\B}+d) \log \frac{1}{\epsilon \cos \theta_0}\right) .
	\end{align*}
	Furthermore, if we use Nesterov's accelerated gradient descent (Algorithm~\ref{algo:AGD}) to solve the linear systems in Algorithm~\ref{algo:appro-power}, the time can be bounded as
	\begin{align*}
	&O\left( \frac{ \nnz{\B} \sqrt{\cn(\B)}}{\rho} \left( \log\frac{1}{\cos \theta_0} \log \frac{1}{\rho \cos \theta_0}
	+ \log \frac{1}{\epsilon} \log \frac{1}{\rho} \right) 
	+ \frac{1}{\rho} \nnz{\A} \log \frac{1}{\epsilon \cos \theta_0}\right).
	\end{align*}
\end{theorem}
\textbf{Remarks}:
\begin{itemize}
	\item	Since \geneig~chooses $\w_0$ randomly, Lemma~\ref{lem:randinit} tells us that $\cos\theta_0 \geq \frac{\zeta}{\sqrt{d\cn\left(\B\right)}}$ with probability greater than $1-\zeta$.
	\item	Note that \geneig~exploits the sparsity of input matrices since we only need to apply them as operators.
	\item	Depending on computational restrictions, we can also use a subset of samples in each iteration of \geneig. In some large scale learning applications using minibatches of data in each iteration helps make the method scalable while still maintaining the quality of performance.
\end{itemize}
\subsection{Top-k Setting} \label{sec:algo-k}
In this section, we give an extension of our algorithm and result for computing the top-$k$ generalized eigenvectors. Our algorithm, \geneigk~is formally given as Algorithm~\ref{algo:appro-ortho}.
\begin{algorithm}[h!]
	\caption{\textbf{Gen}eralized \textbf{E}igenvectors via \textbf{Lin}ear System Solvers-\textbf{K} (\geneigk). }
	\begin{algorithmic}
		\renewcommand{\algorithmicrequire}{\textbf{Input: }}
		\renewcommand{\algorithmicensure}{\textbf{Output: }}
		\REQUIRE $T$, $k$, symmetric matrix $\A$, PSD matrix $\B$. \\
		a subroutine $\textrm{GS}_\B(\cdot)$ that performs Gram-Schmidt process, with inner product $\langle \cdot, \cdot \rangle_{\B}$.
		\ENSURE  top $k$ eigen-space $\W \in \R^{d \times k}$.
		\STATE $\tilde{\W}_0\leftarrow $ random $d \times k$ matrix with each entry i.i.d from $\mathcal{N}(0,1)$
		\STATE $\W_0 \leftarrow \text{GS}_{\B}(\tilde{\W}_0)$.
		\FOR{$t = 0,\cdots,T-1$}
		\STATE $\Gamma_t \leftarrow  (\W_t\trans \B \W_t)^{-1} (\W_t \trans \A \W_t)$
		\STATE $\tilde{\W}_{t+1} \leftarrow \argmin_{\W} \tr(\frac{1}{2}\W\trans\B\W - \W\trans\A\W_t) $ \\
		\quad\quad\quad\{Use an optimization subroutine\\
		\quad\quad\quad~with initialization $\W_t\Gamma_t$ \}
		\STATE $\W_{t+1} \leftarrow \text{GS}_{\B}(\tilde{\W}_{t+1})$
		\ENDFOR
		\STATE \textbf{Return} $\W_T$.
	\end{algorithmic}
	\label{algo:appro-ortho}
\end{algorithm}

\geneigk~is a natural generalization of \geneig~from the previous section. Given an initial set of vectors $\W_0$, the algorithm proceeds by doing approximate orthogonal iteration. Each iteration involves solving $k$ independent linear systems\footnote{Similarly, as before, we could use Nesterov's accelerated gradient descent, i.e. Algorithm~\ref{algo:AGD}.} and orthonormalizing the iterates. The following theorem is the main result of our paper which gives runtime bounds for Algorithm~\ref{algo:appro-ortho}. As before, we assume access to a blackbox linear system solver and also give a result instantiating the theorem with Nesterov's accelerated gradient descent algorithm.
\begin{theorem} \label{thm:main_k}
	Suppose the linear system solver takes time $\Tcal{\delta}$ to reduce the error by a factor $\delta$. Given input matrices $\A$ and $\B$, \geneigk~computes a $d\times k$ matrix $\W_T$ which is an estimate of the top generalized eigenvectors $\V$ with an error of $\epsilon$ i.e., $\W_T \trans \B \W_T = \eye$ and $\sin \theta_T \leq \epsilon$, where $\theta_T \defeq \theta\left(\W_T, \V\right)$ in $T= \frac{2}{\rho} \log \frac{1}{\epsilon \cos\theta_0}$ iterations where $\theta_0 = \theta\left(\W_0,\V\right)$.
	The run time of this algorithm is at most
	\begin{align*}
	&O\left(
	\frac{1}{\rho} \left( \log\frac{1}{\cos \theta_0} \cdot  \Tcal{\frac{\rho^2 \cos^4 \theta_0}{64 k \gamma^2 }} +  \Tcal{\frac{\rho^2}{64 k \gamma^2}} \log \frac{1}{\epsilon}  \right) + \frac{1}{\rho} \left(\nnz{\A} k+\nnz{\B}k + dk^2\right) \log\frac{1}{\epsilon \cos \theta_0}\right) ,
	\end{align*}
	where $\gamma \defeq \frac{|\lambda_1|}{|\lambda_k|}$, $\abs{\lambda_1} \geq \cdots \geq \abs{\lambda_k}$ being the top-$k$ eigenvalues of $\Binv \A$.
	Furthermore, if we use Nesterov's accelerated gradient descent (Algorithm~\ref{algo:AGD}) to solve the linear systems in Algorithm~\ref{algo:appro-ortho}, the time above can be bounded as
	\begin{align*}
	&O\left( \frac{ \nnz{\B}k \sqrt{\cn(\B)}}{\rho} \left( \log\frac{1}{\cos \theta_0} \log \frac{k\gamma}{\rho\cos \theta_0} 
	+\log \frac{1}{\epsilon} \log \frac{k\gamma}{\rho} \right) 
	+ \frac{\left(\nnz{\A}k+dk^2\right)}{\rho} \log\frac{1}{\epsilon \cos \theta_0} \right).
	\end{align*}
\end{theorem}

%
\textbf{Remarks}:
\begin{itemize}
	\item	Lemma~\ref{lem:randinit} again tells us that since $\W_0$ is chosen to be normalized after choosing uniformly at random from the unit sphere, $\cos \theta_0 \geq \frac{\zeta}{\sqrt{d k \cn\left(\B\right)}}$ with probability greater than $1-\zeta$.
	\item	This result recovers Theorem~\ref{thm:main} as a special case, since when $k=1$, we also have $\gamma = \frac{ \abs{\lambda_1}} {\abs{\lambda_1}} = 1$.
\end{itemize}



\section{Application to CCA}\label{sec:cca}
We now outline how the CCA problem can be reduced to computing generalized eigenvectors. The CCA problem is as follows. Given two sets of data points $\X \in \R^{n\times d_1}$ and $\Y \in \R^{n\times d_2}$, let $\S_x \defeq \X\trans \X/n$, $\S_y \defeq \Y\trans \Y/n$, and $\S_{xy} \defeq \X\trans \Y /n$. We wish to find vectors $\bphi_1,\cdots,\bphi_k$ and $\bpsi_1,\cdots,\bpsi_k$ which are defined recursively as
	\begin{align*}
	&\qquad \qquad \left(\bphi_i, \bpsi_i\right) \in \argmax_{\bphi, \bpsi}  \bphi \trans \S_{xy} \bpsi \\
	&\st
	\begin{array}{c}
		\norm{\bphi}_{\S_x} =1 \mbox{ and } \bphi\trans \S_{x} \bphi_j = 0 \; \forall \; j \leq i-1 \\
		\norm{\bpsi}_{\S_y} =1 \mbox{ and } \bpsi\trans \S_{y} \bpsi_j = 0 \; \forall \; j \leq i-1.
	\end{array}	
	\end{align*}
where the values of $\bphi_i\trans\S_{xy}\bpsi_i$ are called canonical correlations between $\X$ and $\Y$.

For reduction, we know any stationary point of this optimization problem satisfies
$\S_{xy} \bpsi_i = \lambda_i \S_{x} \bphi_i$, and $\S_{yx} \bphi_i = \mu_i \S_{y} \bpsi_i$, where $\lambda_i$ and $\mu_i$ are two constants. Combined with the constraints, we also see that $\lambda_i = \mu_i$. This can be written in matrix form as
	$
	\begin{pmatrix}
	0 & \S_{xy} \\
	\S_{yx} & 0
	\end{pmatrix}
	\begin{pmatrix}
	\bphi_i \\ \bpsi_i
	\end{pmatrix}
	= \lambda_i
	\begin{pmatrix}
	\S_x & 0 \\
	0 & \S_y
	\end{pmatrix}
	\begin{pmatrix}
	\bphi_i \\ \bpsi_i
	\end{pmatrix}
	$.
Suppose the generalized eigenvalues of the above matrices are $-\lambda_1 < -\lambda_2 < \cdots < \lambda_2 < \lambda_1$.
The top $2k$-dimensional eigen-space of this generalized eigenvalue problem corresponds to the linear subspace
spanned by the eigenvectors of $\lambda_i$ and $-\lambda_i$, which are
$\begin{pmatrix}
	\bphi_i \\ \bpsi_i
	\end{pmatrix}, 
	\;
	\begin{pmatrix}
	-\bphi_i \\ \bpsi_i
\end{pmatrix} \; \forall \; i \in [k]$.
Once we solve the top-$2k$ generalized eigenvector problem for the matrices $\left(\begin{array}{cc}
0 & \S_{xy} \\ \S_{yx} & 0
\end{array}\right)$ and $\left(\begin{array}{cc}
\S_{xx} & 0 \\ 0 & \S_{yy}
\end{array}\right)$, we can pick any orthonormal basis that spans the output subspace and choose a random $k$-dimensional projection of those vectors. The formal algorithm is given in Algorithm~\ref{algo:appro-cca}. Combining this with our results for computing generalized eigenvectors, we obtain the following result.


\begin{algorithm}[h!]
	\begin{algorithmic}
		\renewcommand{\algorithmicrequire}{\textbf{Input: }}
		\renewcommand{\algorithmicensure}{\textbf{Output: }}
		\REQUIRE $T$, $k$, data matrix $\X \in \R^{n\times d_1}, \Y \in \R^{n\times d_2}$
		\ENSURE  top $k$ canonical subspace $\W_{x} \in \R^{d_1 \times k}, \W_{y} \in \R^{d_2 \times k}$.
		\STATE $\S_{xx} \leftarrow \X\trans\X/n, ~\S_{yy} \leftarrow \Y\trans\Y/n, ~\S_{xy} \leftarrow \X\trans\Y/n$.
		\STATE $\A \leftarrow \left(\begin{array}{cc}
0 & \S_{xy} \\ \S_{xy}\trans & 0
\end{array}\right)$, $\B \leftarrow \left(\begin{array}{cc}
\S_{xx} & 0 \\ 0 & \S_{yy}
\end{array}\right)$
		\STATE $\left(\begin{array}{c}
\bar{\W}_x \in \R^{d_1 \times 2k}\\ \bar{\W}_y\in \R^{d_2 \times 2k} \end{array}\right)\leftarrow $ \geneigk$\left(\A,\B\right)$.
		\STATE $\U \leftarrow $ $2k\times k$ random Gaussian matrix
		\STATE $\tilde{\W}_x \leftarrow \bar{\W}_x \U$.
		\STATE $\tilde{\W}_y \leftarrow \bar{\W}_y \U$.
		\STATE $\W_x = \text{GS}_{\S_{xx}}(\tilde{\W}_x), \W_x = \text{GS}_{\S_{yy}}(\tilde{\W}_y)$
		\STATE \textbf{Return} $\W_x, \W_y$.
	\end{algorithmic}
	\caption{\textbf{CCA} via \textbf{Lin}ear System Solvers (CCALin)}
	\label{algo:appro-cca}
\end{algorithm}

\begin{theorem} \label{thm:main_cca}
	Suppose the linear system solver takes time $\Tcal{\delta}$ to reduce the error by a factor $\delta$. Given inputs $\X$ and $\Y$, with probability greater than $1-\zeta$, then there is some universal constant $c$, so that Algorithm~\ref{algo:appro-cca} outputs $\W_x$ and $\W_y$ such that $\sin \theta(\spanvec{\phi_i;i\in[k]}, \W_x) \le \epsilon$, and $\sin \theta(\spanvec{\psi_i;i\in[k]}, \W_y) \le \epsilon$, in time
	\begin{align*}
	&O\left(
	\frac{1}{\rho} \left( \log\frac{d\cn}{\zeta} \cdot  \Tcal{\frac{c\zeta^6\rho^2 }{d^2k^5\cn^2 \gamma^2 }} +  \Tcal{\frac{c\zeta^2\rho^2}{k^3 \gamma^2}} \log \frac{1}{\epsilon}  \right) + \frac{1}{\rho} \left(\nnz{\X, \Y} k + dk^2\right) \log\frac{d\cn}{\zeta\epsilon}\right) ,
	\end{align*}
	where $\nnz{\X,\Y}\defeq \nnz{\X}+\nnz{\Y}$ and $\kappa \defeq  \max\left(\cn\left(\S_{xx}\right),\cn\left(\S_{yy}\right)\right)$ and $\gamma \defeq \frac{\lambda_1}{\lambda_k}$.
	If we use Nesterov's accelerated gradient descent (Algorithm~\ref{algo:AGD}) to solve the linear systems in GenELink, then the total runtime is
	\begin{align*}
	&O\left(\frac{ \nnz{\X,\Y}k \sqrt{\kappa}}{\rho} 
	\left( \log\frac{d\kappa}{\zeta} \log \frac{d \kappa\gamma}{\zeta \rho }+\log \frac{1}{\epsilon} \log \frac{k \gamma}{\rho} \right) + \frac{dk^2}{\rho} \log\frac{d\kappa}{\zeta \epsilon} \right),
	\end{align*}
\end{theorem}
\textbf{Remarks}:
\begin{itemize}
	\item	Note that we depend on the maximum of the condition numbers of $\S_{xx}$ and $\S_{yy}$ since the linear systems that arise in \geneigk~decompose into two separate linear systems, one in $\S_{xx}$ and the other in $\S_{yy}$.
	\item	We can also exploit sparsity in the data matrices $\X$ and $\Y$ since we only need to apply $\S_{xx}, \S_{xy}$ or $\S_{yy}$ only as operators, which can be done by applying $\X$ and $\Y$ in appropriate order. Exploiting sparsity is crucial for any large scale algorithm since there are many data sets (e.g., URL dataset in our experiments) where dense operations are impractical.
\end{itemize}

\section{Reduction to Linear System}
\label{sec:reduction}

Here we show that solving linear systems in $\B$ is inherent in solving the top-$k$ generalized eigenvector problem in the worst case and we provide evidence a $\sqrt{\kappa(\B)}$ factor in the running time is essential for a broad class of iterative methods for the problem. 

Let $\M$ be a symmetric positive definite matrix and suppose we wish to solve the linear system $\M \x = \m$, i.e. compute $\x_*$ with $\M \x_* = \m$. If we set $\A = \m \m^\top$ and $\B = \M$ then 
\[
\argmax_{\x^\top \B \x = 1} \x^\top \A \x
= \frac{\B^{-1} \m}{\m^\top \B^{-1} \m}
\]
and consequently computing the top-1 generalized eigenvector yields the solution to the linear system. Therefore, the problem of computing top-$k$ generalized eigenvectors is in general harder than the problem of solving symmetric positive definite linear systems. 

Moreover, it is well known that any method which starts at $\m$ and iteratively applies $\M$ to linear combinations of the points computed so far must apply $\M$ at least $\Omega(\sqrt{\kappa(\B)})$ in order to halve the error in the standard norm for the problem \cite{Shewchuk:1994:ICG:865018}.
Consequently, methods that solve the  top-$1$ generalized eigenvector problem by simply applying $\A$ and $\B$, which is the same as applying $\M$ and taking linear combinations with $\m$, must apply $\M$ at least $\Omega(\sqrt{\kappa(\M)})$ times to achieve small error, unless they exploit more structure of $\M$  or the initialization.


\section{Simulations}\label{sec:exps}

In this section, we present our experiment results performing CCA on three benchmark datasets which are summarized in Table~\ref{summary_datasets}. We wish to demonstrate two things via these simulations: 1) the behavior of CCALin verifies our theoretical result on relatively small-scale dataset, and 2) scalability of CCALin comparing it with other existing algorithms on a large-scale dataset.

\begin{table}[h!]
\caption{Summary of Datasets}
\label{summary_datasets}
\vskip 0.15in
\begin{center}
\begin{small}
\begin{sc}
\begin{tabular}{ccccc}
\hline
\abovespace\belowspace
Dataset & $d_1$ & $d_2$ & $n$ & sparsity
\footnotemark\\
\hline
\abovespace
MNIST    & $392$ & $392$ & $6\times 10^4$ & $0.19$\\
Penn Tree Bank    &  $10^4$ & $10^4$ & $5 \times 10^5$ & $1\times 10^{-4}$\\
URL Reputation    &  $10^5$ & $10^5$ & $1 \times 10^6$ & $5.8 \times 10^{-5}$ \\
\hline
\end{tabular}
\end{sc}
\end{small}
\end{center}
\vskip -0.1in
\end{table}
\footnotetext{Sparsity is given by $(\nnz{\X} + \nnz{\Y})/(nd_1+nd_2)$.}

Let us now specify the error metrics we use in our experiments. The first ones are the principal angles between the estimated subspaces and the true ones. Let $\W_x$ and $\W_y$ be the estimated subspaces and $\V_x$, $\V_y$ be the true canonical subspaces. We will use principle angles $\theta_x = \theta(\W_x, \V_x)$ under $\S_{xx}$-norm, $\theta_y = \theta(\W_x, \V_x)$ under $\S_{yy}$-norm and
$\theta_{\B} = \theta\left(
\left(\begin{array}{c c}
\V_x & 0\\ 0 & \V_y\end{array}\right),
\left(\begin{array}{c c}
\bar{\W}_x \\ \bar{\W}_y\end{array}\right)\right)$
\footnote{See Algorithm \ref{algo:appro-cca} for definition of $\bar{\W}_x, \bar{\W}_y$}, under the $\left(\begin{array}{c c}
\S_{xx} & 0\\ 0 & \S_{yy}\end{array}\right)$ norm.
Unfortunately, we cannot compute these error metrics for large-scale datasets since they require knowledge of the true canonical components. Instead we will use Total Correlations Captured (TCC), which is another metric widely used by practitioners, defined to be the sum of canonical correlation between two matrices. Also, Proportion of Correlations Captured is given as
$$\text{PCC} = \text{TCC}(\X\W_x, \Y\W_y)/\text{TCC}(\X\V_x, \Y\V_y)$$
For a fair comparison with other algorithms (which usually call highly optimized matrix inversion subroutines), we use number of FLOPs instead of wall clock time to measure the performance.

\subsection{Small-scale Datasets}
\textbf{MNIST} dataset\cite{lecun1998mnist} consists of 60,000 handwritten digits from 0 to 9.
Each digit is a image represented by $392\times392$ real values in [0,1]. 
Here CCA is performed between left half images and right half images.
The data matrix is dense but the dimension is fairly small.

\noindent\textbf{Penn Tree Bank} (PTB) dataset comes from full Wall Street Journal Part of Penn Tree Bank which consists of 1.17 million tokens and a vocabulary size of 43k\cite{marcus1993building}, which has already been used to successfully learn the word embedding by CCA\cite{dhillon2011multi}. Here, the task is to learn correlated components between two consecutive words. We only use the top 10,000 most frequent words. Each row of data matrix $\X$ is an indicator vector and hence it is very sparse and $\X\trans\X$ is diagonal.

Since the input matrices are very ill conditioned, we add some regularization and replace $\S_{xx}$ by $\S_{xx}+\lambda\I$ (and similarly with $\S_{yy}$). In CCALin, we run \geneigk~with $k=10$ and accelerated gradient descent
 (Algorithm \ref{algo:AGD} in the supplementary material) to solve the linear systems. 
The results are presented in Figure \ref{fig:mnist_ptb} and Figure \ref{fig:mnist_ptb_2}.
\begin{figure}[ht]
\begin{center}

\subfigure{\includegraphics[width=0.48\columnwidth]{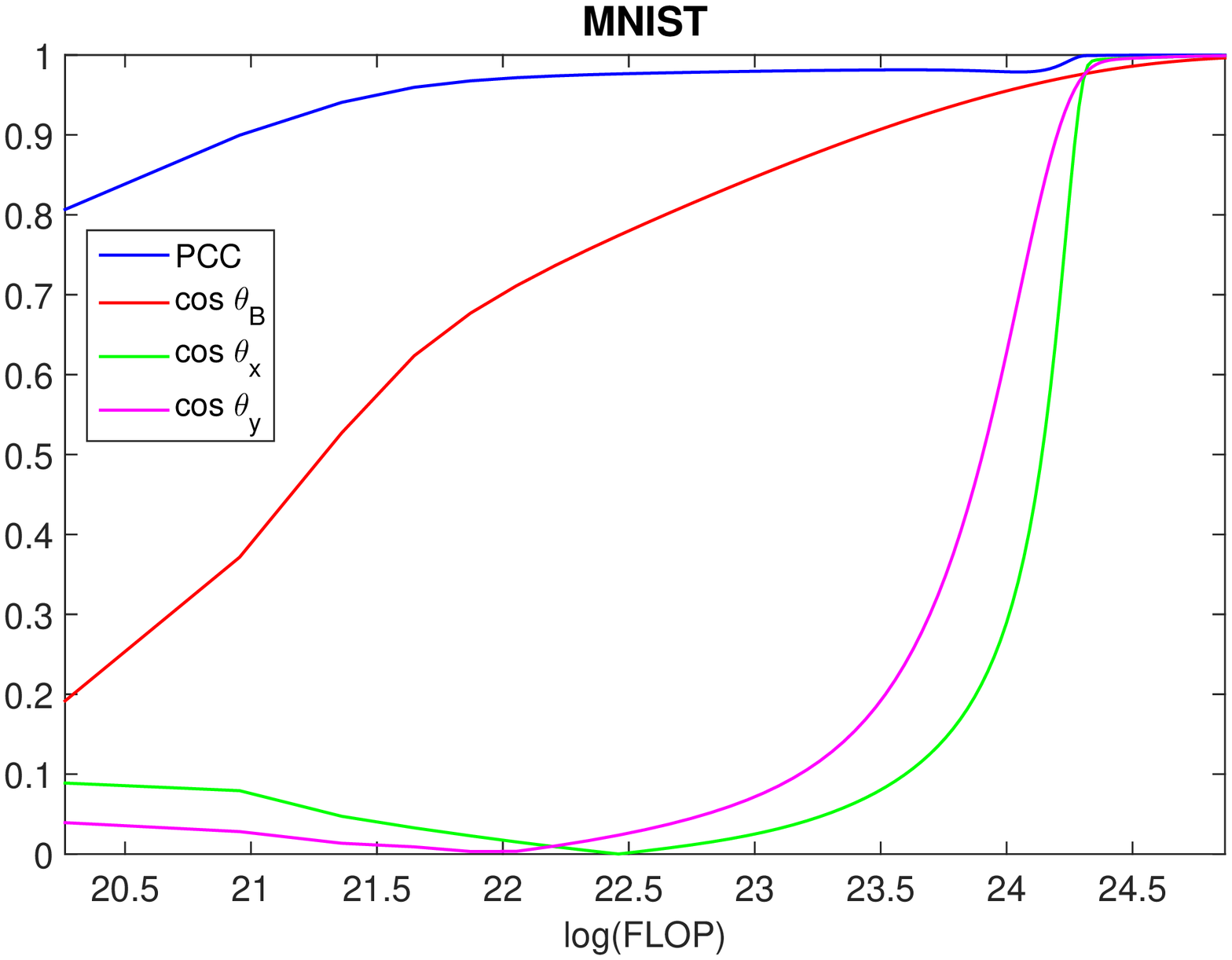}}
\subfigure{\includegraphics[width=0.48\columnwidth]{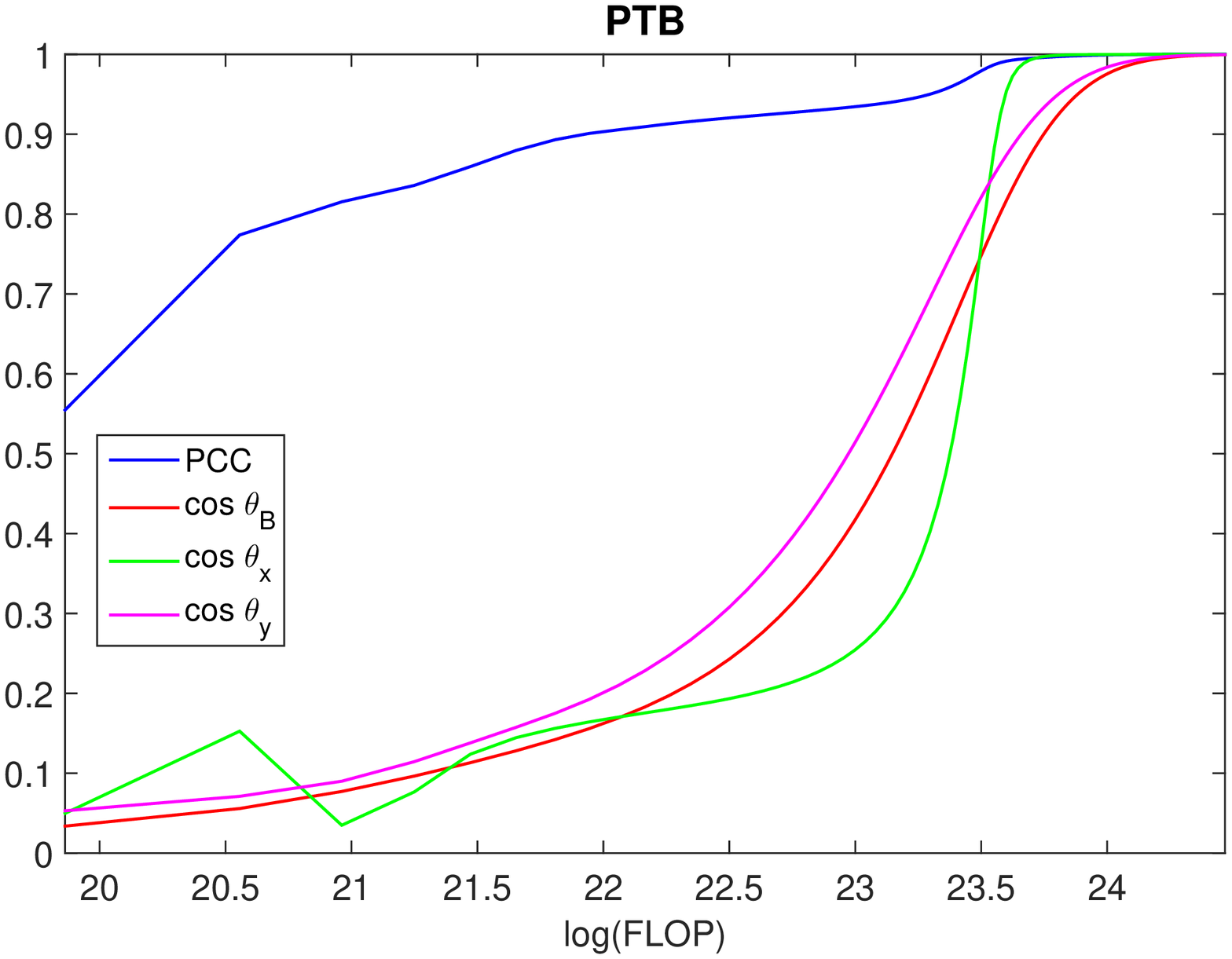}}
\caption{Global convergence of PCC and principle angles on MNIST and PTB Datasets}
\label{fig:mnist_ptb}
\end{center}
\vskip -0.1in
\end{figure} 

Figure \ref{fig:mnist_ptb} shows a typical run of CCALin from random initialization on both 
MNIST and PTB dataset. We see although $\theta_x, \theta_y$ may be even 90 degree at some point respectively, $\theta_\B$ is always monotonically decreasing (as $\cos \theta_\B$ monotonically increasing) as predicted by our theory. In the end, as $\theta_\B$ goes to zero, it will push both $\theta_x$ and $\theta_y$ go to zero, and PCC go to 1. This demonstrates that our algorithm indeed converges to the true canonical space. 

\begin{figure}[ht]
\begin{center}

\subfigure{\includegraphics[width=0.48\columnwidth]{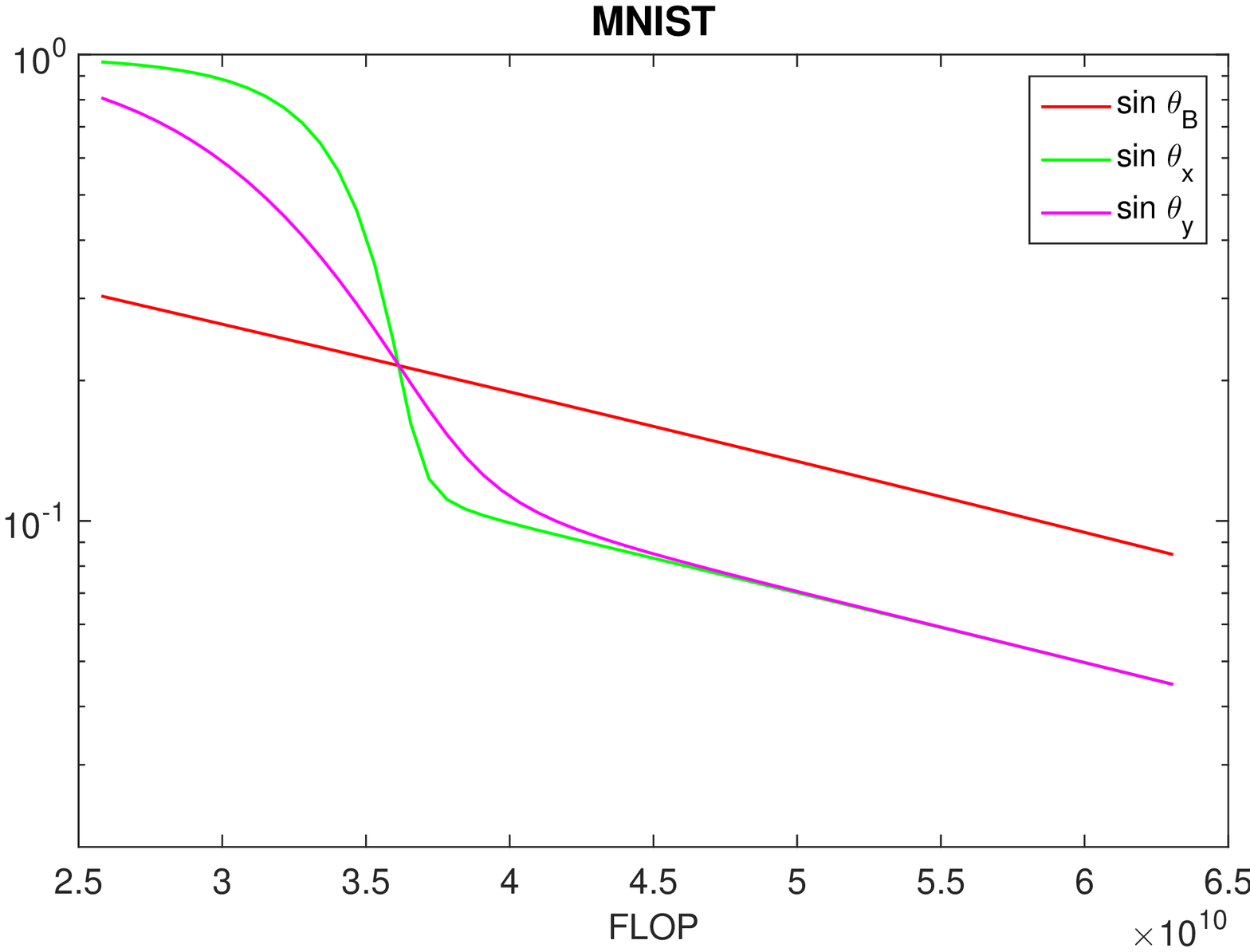}}
\subfigure{\includegraphics[width=0.48\columnwidth]{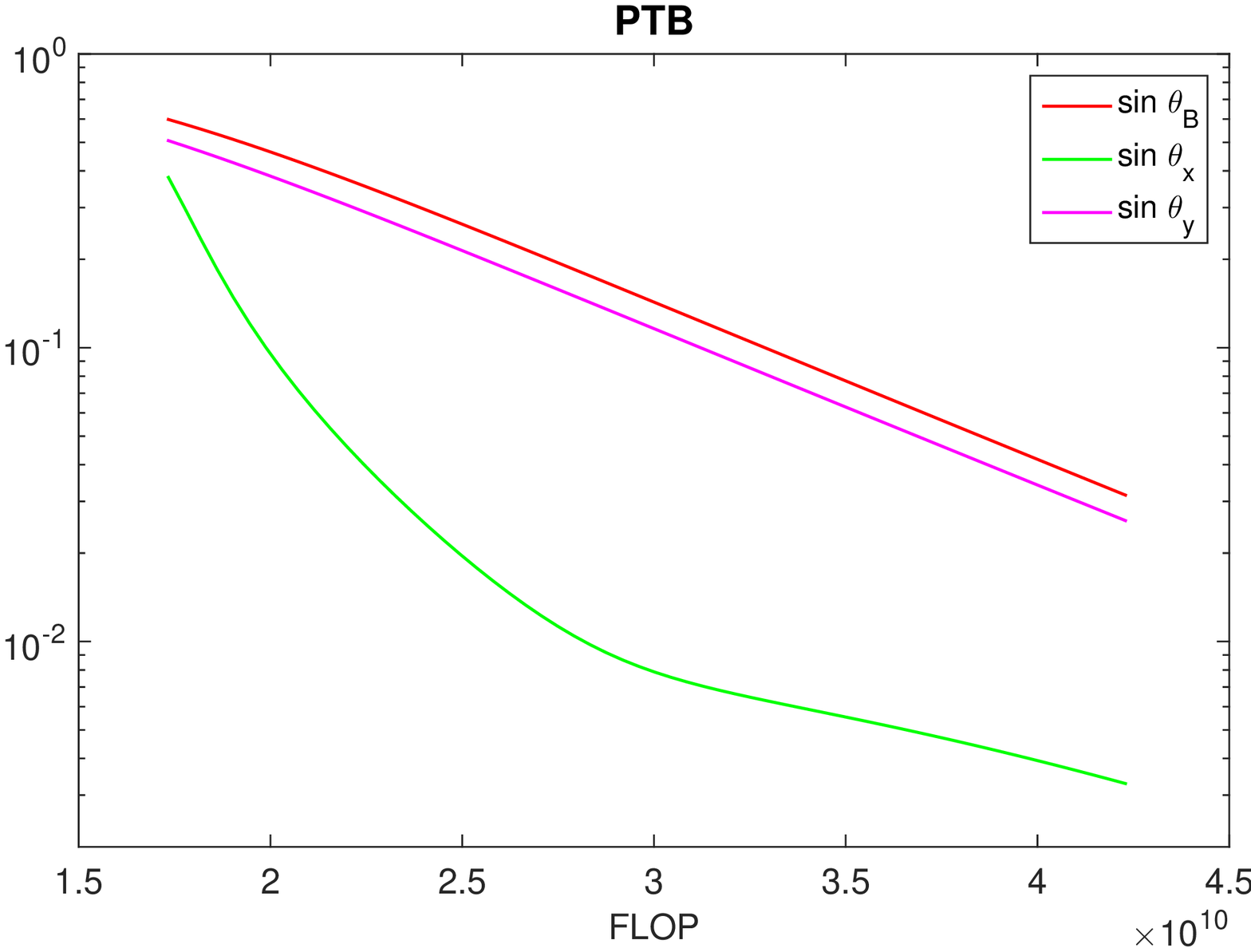}}
\caption{Linear convergence of principle angles on MNIST and PTB Datasets}
\label{fig:mnist_ptb_2}
\end{center}
\vskip -0.1in
\end{figure} 

Furthermore, by a more detailed examination of experimental data in Figure \ref{fig:mnist_ptb}, we observe in Figure \ref{fig:mnist_ptb_2} that $\sin \theta_\B$ is indeed linearly convergent as we predicted in the theory. 
In the meantime, $\sin \theta_x$ and $\sin\theta_y$ may initially converge a bit slower than $\sin\theta_\B$, but in the end they will be upper bounded by $\sin\theta_\B$ times a constant factor, thus will eventually converge at a linear rate at least as fast as $\sin\theta_\B$.


\subsection{Large-scale Dataset}

\textbf{URL Reputation} dataset contains 2.4 million URLs and 3.2 million features including both host-based features and lexical based features. Each feature is either real valued or binary. For experiments in this section, we follow the setting of~\cite{deanCCA}. We use the first 2 million samples, and run CCA between a subset of host based features and a subset of lexical based features to extract the top $20$ components. Although the data matrix $\X$ is relatively sparse, unlike PTB, it has strong correlations among different coordinates, which makes $\X\trans\X$ much denser ($\nnz{\X\trans\X}/d_1^2 \approx 10^{-3}$). 

Classical algorithms are impractical for this dataset on a typical computer, either running out of memory or requiring prohibitive amount of time. Since we cannot estimate the principal angles, we will evaluate TCC performance of CCALin.

We compare our algorithm to S-AppGrad~\cite{deanCCA} which is an iterative algorithm and PCA-CCA~\cite{deanCCA}, NW-CCA~\cite{WittenTH2009} and DW-CCA~\cite{lu2014large} which are one-shot estimation procedures.

\begin{figure}[ht]
\vspace{0.2in}
\begin{center}
\centerline{\includegraphics[width=0.5\columnwidth]{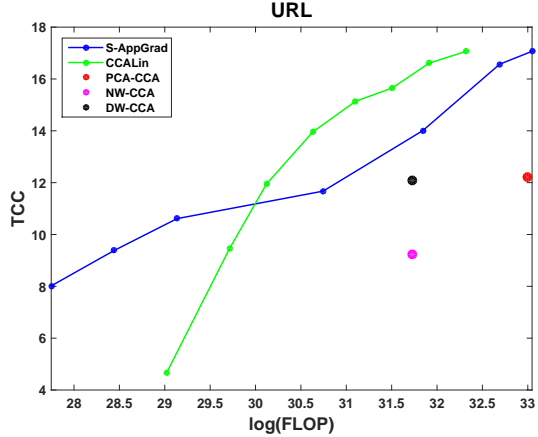}}
\caption{Comparison with existing algorithms on URL dataset}
\label{fig:url}
\end{center}
\vspace{-0.2in}
\end{figure} 

In CCALin, we employ \geneigk~using stochastic accelerated gradient descent for solving linear systems using minibatches in each of the gradient steps and also leverage sparsity of the data to deal with the large data size. The result is shown in Figure~\ref{fig:url}.
It is clear from the plot that our algorithm takes fewer computations than the other algorithms to achieve the same accuracy.

\section{Conclusion}

In summary, we have provided the first provable globally linearly convergent algorithms for solving canonical correlation analysis and the generalized eigenvector problems. We have shown that for recovering the top $k$ components our algorithms are much faster than traditional methods based on fast matrix multiplication and singular value decomposition when $k \ll n$ and the condition numbers and eigenvalue gaps of the matrices involved are moderate. Moreover, we have provided empirical evidence that our algorithms may be useful in practice. We hope these results serve as the basis for further improvements in performing large scale data analysis both in theory and in practice.

\bibliographystyle{apalike}
\bibliography{refs}

\newpage
\appendix

\section{Solving Linear System via Accelerated Gradient Descent}

 \begin{algorithm}[h!]
 \caption{Nesterov's accelerated gradient descent}\label{algo:AGD}
 \begin{algorithmic}
 \renewcommand{\algorithmicrequire}{\textbf{Input: }}
 \renewcommand{\algorithmicensure}{\textbf{Output: }}
 \REQUIRE learning rate $\eta$, factor $Q$, initial point $\x_0$, $T$.
 \ENSURE  minimizer $x^\star$ of $f$ .
 \FOR{$t = 0,\cdots,T-1$}
 \STATE $\y_{t+1} \leftarrow \x_t - (1/\beta)\cdot \grad f(\x_t)$
 \STATE $\x_{t+1} \leftarrow  \y_{t+1} +  (\sqrt{Q}-1)/(\sqrt{Q}+1) \cdot(\y_{t+1} - \y_t) $
 \ENDFOR
 \STATE \textbf{Return} $\y_T$.
 \end{algorithmic}
 \end{algorithm}

 Since we use accelerated gradient descent in our main theorems, for completeness, we put the algorithm and cite its result about iteration complexity here without proof.

\begin{theorem}[\cite{nesterov1983method}] \label{thm:AGD_guarantee}
	Let $f$ be $\alpha$-strongly convex and $\beta$-smooth, then accelerated gradient descent with learning rate $\eta = \frac{1}{\beta}$ and $Q = \beta /\alpha$ satisfies:
	\begin{equation}
	f(\x_t) - f(\x^\star) \le 2(f(\x_0) - f(\x^\star))\exp(-\frac{t}{\sqrt{Q}})
	\end{equation}
\end{theorem}


\section{Proofs of Main Theorem}
In this section we will prove Theorems~\ref{thm:main}, ~\ref{thm:main_k} and ~\ref{thm:main_cca}.
\subsection{Rank-1 Setting}
We first prove our claim that $\Binv \A$ has an eigenbasis.
\begin{lemma}\label{lem:eig-sing}
	Let $(\u_i,\sigma_i)$ be the eigenpairs of the symmetric matrix $\B^{-1/2}\A\B^{-1/2}$. Then $\B^{-1/2} \u_i$ is an eigenvector of $\Binv \A$ with eigenvalue $\sigma_i$.
\end{lemma}
\begin{proof}
	The proof is straightforward.
	\begin{align*}
		\Binv \A \left(\B^{-1/2} \u_i\right)
		&= \B^{-1/2}\left(\B^{-1/2}\A\B^{-1/2} \u_i\right) = \sigma_i \B^{-1/2} \u_i.
	\end{align*}
\end{proof}
Denote the eigenpairs of $\Binv\A$ by $(\lambda_i,\v_i)$,
the above lemma further tells us that $\v_i \trans \B \v_j = \u_i \trans \u_j = \delta_{ij}$.

Recall that we defined the angle between $\w$ and $\v_1$ in the $\B$-norm: $\theta\left(\w,\v_1\right)=\arccos\left(|\v_1\trans \B \w|\right)$.

To measure the distance from optimality, we use the following potential function for normalized vector $\w$ ($\norm{\w}_{\B} = 1$):
\begin{align}
\tan\theta(\w, \v_1)  = \frac{\sqrt{1 - |\v_1\trans \B \w|^2 }}{|\v_1\trans \B \w| }.
\end{align}


\begin{lemma}\label{lem:potfunc}
	Consider any $\w$ such that $\norm{\w}_{\B}=1$ and $\tan\theta(\w, \v_1)\leq \epsilon$. Then, we have:
	\begin{align*}
		\cos^2\theta(\w, \v_1) = (\v_1 \trans \B \w)^2 \geq 1-\epsilon^2 \mbox{ and } \w \trans \A \w \geq \lambda_1(1-\epsilon^2).
	\end{align*}
\end{lemma}
\begin{proof}
	Clearly, 
	\begin{align*}
		(\v_1 \trans \B \w)^2 = \cos^2\theta(\w, \v_1) = \frac{1}{1+\tan^2\theta(\w, \v_1)}
	  \geq \frac{1}{1+\epsilon^2} \geq 1-\epsilon^2,
	\end{align*}
	proving the first part. For the second part, we have the following:
	\begin{align*}
		\w \trans \A \w & = \sum_{i,j} (\v_i \trans \B \w)(\v_j \trans \B \w) \v_i \trans \A \v_j = \sum_{i,j} \lambda_j (\v_i \trans \B \w)(\v_j \trans \B \w) \v_i \trans \B \v_j \\
		& = \sum_i \lambda_i (\v_i \trans \B \w)^2 \geq \lambda_1 (\v_1 \trans \B \w)^2 \geq (1-\epsilon^2)\lambda_1,
	\end{align*}
	proving the lemma.
\end{proof}

\begin{proof}[Proof of Theorem~\ref{thm:main}]
	We will show that the potential function $\tan \theta(\w_t, \v_1)$ decreases geometrically with $t$. This will directly provides an upper bound for $\sin \theta(\w_t, \v_1)$. For simplicity, through out the proof we will simply denote $\theta(\w_t, \v_i)$ as $\theta_t$.

	Recall the updates in Algorithm \ref{algo:appro-power}, suppose at time $t$, we have $\w_t$ such that $\norm{\w_t}_{\B}=1$. Let us say
	\begin{equation}
	\w_{t+1} = \frac{1}{Z}(\B^{-1}\A \w_t + \xi)
	\end{equation}
	where $Z$ is some normalization factor, and $\xi$ is the error in solving the least squares. We will first prove the geometric convergence claim assuming
	\begin{align}
		\norm{\xi}_\B \le \frac{|\lambda_1| - |\lambda_2|}{4}\min\{ \cos\theta_t, \sin\theta_t\}, \label{eqn:xibound}
	\end{align}
	and then bound the time taken by black-box linear system solver to provide such an accuracy. Since $\w_t$ can be written as $\w_t = \sum_i \left(\w_t \trans \B \v_i\right) \v_i$, we know $\B^{-1}\A \w_t = \sum_{i=1}^d  \lambda_i \left(\w_t \trans \B \v_i\right) \v_i$.
	Since $\norm{\w_{t+1}}_{\B}=1$ and $\v_i \trans \B \v_j = \delta_{ij}$, we have
	\begin{align*}
	\tan\theta_{t+1} &= \frac{\sqrt{Z^2 - |\v_1\trans \B Z\w_{t+1}|^2 }}{|\v_1\trans \B Z\w_{t+1}| }
	\le \frac{\sqrt{\sum_{i=2}^d \left(\w_t \trans \B \v_i\right)^2 \lambda_i^2 } +\norm{\xi}_\B }{|\left(\w_t \trans \B \v_1\right) \lambda_1| -\norm{\xi}_\B} \nn\\
	&\le \frac{\sqrt{1-\left(\w_t \trans \B \v_1\right)^2}}{|\w_t \trans \B \v_1|}\times\frac{|\lambda_2| + \frac{\norm{\xi}_\B }{\sqrt{1-\left(\w_t \trans \B \v_1\right)^2}} }{|\lambda_1| -\frac{\norm{\xi}_\B}{|\w_t \trans \B \v_1|} }
	= \tan\theta_t \times \frac{|\lambda_2| + \frac{\norm{\xi}_\B }{\sqrt{1-\left(\w_t \trans \B \v_1\right)^2}} }{|\lambda_1| -\frac{\norm{\xi}_\B}{|\w_t \trans \B \v_1|} }
	\end{align*}

	By definition of $\theta_t$, we know $\cos \theta_t = |\w_t \trans \B \v_1|$ and $\sin \theta_t = \sqrt{1-\left(\w_t \trans \B \v_1\right)^2}$ giving us
	\begin{align*}
	\tan\theta_{t+1} \leq \tan\theta_t \times \frac{|\lambda_2| + \frac{\norm{\xi}_\B }{\sin \theta_t} }{|\lambda_1| -\frac{\norm{\xi}_\B}{\cos \theta_t} }.
	\end{align*}

	Since $\norm{\xi}_\B \le \frac{|\lambda_1| - |\lambda_2|}{4}\min\{ \cos\theta_t, \sin \theta_t\}$, we have that
	\begin{align*}
	\tan \theta_{t+1} \le 
	\frac{|\lambda_1| + 3|\lambda_2|}{3|\lambda_1| + |\lambda_2|} \times \tan\theta_{t}.
	\end{align*}
	Letting $\gamma = \frac{3|\lambda_1| + |\lambda_2|}{|\lambda_1| + 3|\lambda_2|}$, this shows that $G(\w_{t}) \leq \gamma^t G(\w_0)$. Recalling the definition of eigengap $\rho = 1 - \frac{\abs{\lambda_2}}{\abs{\lambda_1}}$, choosing $t$ to be
	\begin{align}
		t \geq \frac{2}{\rho}\log\left(\frac{1}{\epsilon\cos\theta_0}\right)
		\geq \frac{\log\left(\frac{\tan\theta_0}{\epsilon}\right)}{\left(\frac{1}{\gamma}-1\right)}
		\geq \frac{\log\left(\frac{\tan\theta_0}{\epsilon}\right)}{\log\left(\frac{1}{\gamma}\right)}, \label{eqn:outerloop}
	\end{align}
	we are guaranteed that $\sin\theta_t \le \tan\theta_t \leq \epsilon$.
	This number of iterations $\frac{2}{\rho}\log\left(\frac{1}{\epsilon\cos\theta_0}\right)$ could be further decompose into two phase: 1) initial phase $\frac{2}{\rho}\log\frac{1}{\cos\theta_0}$ which mainly caused by large initial angle, 2) convergence phase $\frac{2}{\rho}\log\frac{1}{\epsilon}$ which is mainly due to the high accuracy $\epsilon$ we need.

	We now focus on how to obtain the iterate $\w_{t+1}$ using accelerated gradient descent such that the error $\xi$ has norm bounded as in~\eqref{eqn:xibound}.

	Let $f(\w) \defeq \frac{1}{2} \w \trans \B \w - \w \trans \A \w_t$ and recall that in each iteration, we use linear system solver to solve the following optimization problem:
	\begin{align}
		\min_{\w} f(\w).\label{prob:innloopopt}
	\end{align}
	The minimizer of \eqref{prob:innloopopt} is $\Binv \A \w_t$.
	Define $\epsilon_{\textrm{init}}$ and $\epsilon_{\textrm{des}}$ as initial error and required destination error of linear system solver $\norm{\w - \B^{-1}\A\w_t}^2_\B$.
	Observe that for any $\w$ we have equality,
	\begin{equation}
	 \norm{\w - \B^{-1}\A\w_t}^2_\B = 2(f(\w) - f(\B^{-1}\A\w_t))
	\end{equation}
	Eq.(\ref{eqn:xibound}) directly poses a condition on $\epsilon_{\textrm{des}}$:
	\begin{equation*}
	\epsilon_{\textrm{des}} \le \frac{(|\lambda_1| - |\lambda_2|)^2}{16}\min\{ \cos^2 \theta_t,   \sin^2 \theta_t \}
	\end{equation*}

	Since we initialize Algorithm~\ref{algo:AGD} with $\beta_t \w_t$, where $\beta_t \defeq \frac{\w_t \trans \A \w_t}{\w_t \trans \B \w_t}$, the initial error can be bounded as follows:
	\begin{align*}
	\epsilon_{\textrm{init}} &= 2( f(\beta_t \w_t) - f(\Binv \A \w_t)) \\
	&= 2(\min_\beta f(\beta \w_t) - f(\Binv \A \w_t))
	\leq 2( f(\lambda_1 \w_t) - f(\Binv \A \w_t)) \\
	&= \norm{\lambda_1 \w_t - \Binv \A \w_t }_{\B}^2 \\
	&= \sum_{i\geq 2} (\lambda_1 - \lambda_i)^2\left(\w_t \trans \B \v_i \right)^2
	\leq \lambda_1^2 (1-\left(\w_t \trans \B \v_1 \right)^2) = \lambda_1^2 \sin^2 \theta_t.
	\end{align*}
	This means that we wish to decrease the ratio of final to initial error smaller than
	\begin{align}
		\frac{\epsilon_{\textrm{des}}}{\epsilon_{\textrm{init}}}
		&\leq \frac{(|\lambda_1| - |\lambda_2|)^2}{16}\min\{ \cos^2 \theta_t,   \sin^2 \theta_t \} \times \frac{1}{\lambda_1^2 \sin^2 \theta_t}
		= \frac{\rho^2}{16} \min\left\{\frac{1}{\tan^2\theta_t},1\right\}.\label{eqn:innlooperror}
	\end{align}
	Recall we defined $\Tcal{\delta}$ as the time for linear system solver to reduce the error by a factor $\delta$. Therefore, in the initial phase where $\theta_t$ is large, 
	it would be suffice to solve linear system up to factor $\delta = \frac{\rho^2\cos^2\theta_0}{16} 
	\le \frac{\rho^2}{16\tan\theta_t}$. In convergence phase, where $\theta_t$ is small, choose $\delta = \frac{\rho^2}{16}$ would be sufficient. 

	Therefore, adding the computational cost of Algorithm \ref{algo:appro-power} other than by linear system solver, it's not hard to get the total running time will be bounded by
	\begin{align*}
	&\frac{2}{\rho} \left( \log\frac{1}{\cos \theta_0} \cdot  \Tcal{\frac{\rho^2\cos^2\theta_0}{16} }+\log \frac{1}{\epsilon} \cdot \Tcal{\frac{\rho^2}{16}} \right) + \frac{2}{\rho} (\nnz{\A}+\nnz{\B}+d) \log \frac{1}{\epsilon \cos \theta_0}.
	\end{align*}

	Furthermore, if we run Nesterov's accelerated gradient descent (Algorithm~\ref{algo:AGD}) on function $f(\w)$ to solve the linear systems. Since the condition number of the optimization problem~\eqref{prob:innloopopt} is $\cn(\B)$, by Theorem \ref{thm:AGD_guarantee}, we know $\Tcal{\delta} = O(\nnz{\B}\sqrt{\cn(\B)}\log \frac{1}{\delta})$. Substituting this gives runtime:
	\begin{align*}
	&O\left( \frac{ \nnz{\B} \sqrt{\cn(\B)}}{\rho} \left( \log\frac{1}{\cos \theta_0} \log \frac{1}{\rho \cos \theta_0}+\log \frac{1}{\epsilon} \log \frac{1}{\rho} \right) + \frac{1}{\rho} \nnz{\A} \log \frac{1}{\epsilon \cos \theta_0}\right).
	\end{align*}
	which finishes the proof.
\end{proof}

\subsection{Top-k Setting}
To prove the convergence of subspace, we need a notion of angle between subspaces. The standard definition the is principal angles.
\begin{definition}[Principal angles]
Let $\mathcal{X}$ and $\mathcal{Y}$ be subspaces of $\R^d$ of dimension at least $k$. The principal angles $0 \le \theta^{(1)} \le \cdots \le \theta^{(k)}$ between $\mathcal{X}$ and $\mathcal{Y}$ with respect to $\B$-based scalar product are defined recursively via:
\begin{align*}
\theta^{(i)}(\mathcal{X}, \mathcal{Y}) = \min \{\arccos(\frac{\langle \x, \y \rangle_\B}{\norm{\x}_\B \norm{\y}_\B}):
\x \in \mathcal{X}, \y \in \mathcal{Y}, \x \perp_\B \x_{j}, \y \perp_\B \y_{j} \text{~for all~} j<i \} \\
(\x_{i},\y_{i}) \in \argmin \{\arccos(\frac{\langle \x, \y \rangle_\B}{\norm{\x}_\B \norm{\y}_\B}):
\x \in \mathcal{X}, \y \in \mathcal{Y}, \x \perp_\B \x_{j}, \y \perp_\B \y_{j} \text{~for all~} j<i \}
\end{align*}
For matrices $\X$ and $\Y$, we use $\theta_{j}(\X, \Y)$ to denote the $j$-th principal angle between their range.
\end{definition}

Since for our interest, we only care the largest principal angle, thus, in the following proof, without ambiguity, for $\X, \Y \in \R^{d\times k}$, we use $\theta(\X, \Y)$ to indicate $\theta^{(k)}(\X, \Y)$.
Next lemma will tells us this definition of $\theta(\X, \Y)$ to be the largest principal angle is same as what we defined in the main paper Definition \ref{def:lp_angle}.

\begin{lemma}
Let $\X, \Y \in \R^{d\times k}$ be orthonormal bases (w.r.t $\B$) for subspace $\mathcal{X}, \mathcal{Y}$ respectively. Let $\X_\perp$ be an orthonormal basis for orthogonal complement of $\mathcal{X}$ (w.r.t $\B$).
Then we have 
\begin{equation}
\cos \theta(\mathcal{X}, \mathcal{Y}) = \sigma_{k}(\X\trans \B \Y), \quad
\sin \theta(\mathcal{X}, \mathcal{Y}) = \norm{\X_\perp\trans \B\Y}
\end{equation}
and assuming $\X\trans \B \Y$ is invertible ($\theta(\mathcal{X}, \mathcal{Y})<\frac{\pi}{2}$), we have:
\begin{equation}
\tan \theta(\mathcal{X}, \mathcal{Y}) = \norm{\X_\perp \trans \B \Y(\X \trans \B \Y)^{-1}}
\end{equation}
\end{lemma}

\begin{proof}
By definition of principal angle, it's easy to show $\cos \theta(\mathcal{X}, \mathcal{Y}) = \sigma_{k}(\X\trans \B \Y)$. The projection operator onto subspace $\mathcal{X}$ is $\X\X\trans \B$. It's also easy to show $\X\X\trans \B + \X_\perp \X_\perp\trans \B = \I$
Then, we have:
\begin{align}
&(\X_\perp\trans \B\Y)\trans \X_\perp\trans \B\Y
=\Y\trans \B \X_\perp \X_\perp\trans \B\Y \nn \\
&= \Y\trans \B (\I - \X\X\trans \B)\Y
= \Y\trans \B \Y - (\X\trans \B\Y)\trans(\X\trans \B\Y)
= \I - (\X\trans \B\Y)\trans(\X\trans \B\Y)
\end{align}
Therefore:
\begin{equation}
\norm{\X_\perp\trans \B\Y}^2 = 1-\sigma^2_{k}(\X\trans \B \Y) = 1-\cos^2 \theta(\mathcal{X}, \mathcal{Y})
=\sin^2 \theta(\mathcal{X}, \mathcal{Y})
\end{equation}
Similarily:
\begin{align}
&[\X_\perp \trans \B \Y(\X \trans \B \Y)^{-1}]\trans \X_\perp \trans \B \Y(\X \trans \B \Y)^{-1} \nn \\
=&[(\X \trans \B \Y)^{-1}]\trans [\I - (\X\trans \B\Y)\trans(\X\trans \B\Y)](\X \trans \B \Y)^{-1} \nn \\
=& [(\X \trans \B \Y)^{-1}]\trans (\X \trans \B \Y)^{-1} - \I
\end{align}
Therefore:
\begin{equation}
\norm{\X_\perp \trans \B \Y(\X \trans \B \Y)^{-1}}^2
= \frac{1}{\sigma^2_{k}(\X\trans \B \Y)} - 1=
\frac{1}{\cos^2 \theta(\mathcal{X}, \mathcal{Y})} - 1 = \tan^2 \theta(\mathcal{X}, \mathcal{Y})
\end{equation}
Obviously, $\theta(\mathcal{X}, \mathcal{Y})$ is acute, thus $\sin \theta(\mathcal{X}, \mathcal{Y})>0$
and $\tan \theta(\mathcal{X}, \mathcal{Y}) >0$, which finishes the proof.
\end{proof}

Similar to the top one case, for simplicity, we denote $\theta_t \defeq \theta\left(\W_t,\V\right)$, where $\V \in \R^{d\times k}$ is top $k$ eigen-vector of generalized eigenvalue problem. Now we are ready to prove the theorem. We also denote $\Vc \in \R^{d\times (d-k)}$.
Also throughout the proof, for any matrix $\X$, we use notation $\norm{\X}_\B \equiv \norm{\B^{\frac{1}{2}}\X} \equiv \sqrt{\norm{\X\trans \B \X}}$ and $\norm{\X}_{\B, F} = \norm{\B^\frac{1}{2}\X}_F = \sqrt{\tr(\X\trans \B \X)}$.


\begin{proof}[Proof of Theorem~\ref{thm:main}]
Let $\V \in \R^{d\times k}, \Lambda_\V \in \R^{k\times k}$ be the top $k$ generalized eigen-pairs; and $\Vc \in \R^{d\times (d-k)}, \Lambda_{\Vc} \in \R^{(d-k) \times (d-k)}$ be the remaining $(d-k)$ generalized eigen-pairs (assume all eigen-vectors normalized w.r.t. $\B$). Then, we have:
\begin{align*}
\A &=  \B(\V \Lambda_\V \V\trans + \Vc \Lambda_{\Vc} \Vc\trans)\B \\
\B &=  \B(\V\V\trans + \Vc\Vc\trans)\B
\end{align*}

By approximately solving $\argmin_{\W \in \R^{d\times k}} \tr(\frac{1}{2}\W\trans\B\W - \W\trans\A\W_t)$
and Gram-Schmidt process, we have:
\begin{equation}
\W_{t+1} = (\B^{-1}\A \W_t + \xi)\mat{R}
\end{equation}
where $\mat{R} \in \R^{k \times k}$ is an invertable matrix generated by Gram-Schmidt process.

We will follow the same strategy as in top 1 case, which will first prove the geometric convergence of $\tan\theta_t$ assuming
\begin{equation}
\norm{\xi}_\B \le \frac{|\lambda_k| - |\lambda_{k+1}|}{4}\min\{ \sin\theta_t, \cos\theta_t\}
\label{xi_requirement_k}
\end{equation}
Note here $\xi$ is a matrix, and $\norm{\xi}_\B = \norm{\B^{\frac{1}{2}}\xi} = \sqrt{\norm{\xi\trans \B \xi}}$.
Then we will bound the time taken by black-box linear system solver to provide such an accuracy.

By definition of $\tan\theta_t$ and linear algebra calculation, we have
\begin{align}
\tan \theta_{t+1} &= \norm{\Vc\trans \B \W_{t+1} (\V\trans \B \W_{t+1})^{-1}}\nn\\
& = \norm{\Vc\trans \B \tilde{\W}_{t+1} (\V\trans \B \tilde{\W}_{t+1})^{-1}} \nn\\
& = \norm{(\Lambda_{\Vc}\Vc\trans \B\W_{t} + \Vc\trans \B \xi)(\Lambda_{\V}\V\trans \B \W_{t} + \V\trans \B\xi)^{-1}} \nn\\
&\le  \frac{\norm{(\Lambda_{\Vc}\Vc\trans \B\W_{t} + \Vc\trans \B \xi)(\V\trans \B \W_{t})^{-1}}}{\sigma_k(\Lambda_{\V} + \V\trans \B\xi(\V\trans \B \W_{t})^{-1})} \nn\\
&\le
\frac{\norm{\Lambda_{\Vc}} \tan \theta_t + \norm{\Vc\trans \B \xi(\V\trans \B \W_{t})^{-1}}}
{\sigma_k(\Lambda_{\V}) - \norm{\V\trans \B\xi(\V\trans \B \W_{t})^{-1}}} \nn\\
&\le
\frac{\norm{\Lambda_{\Vc}} \tan \theta_t + \norm{\Vc\trans \B \xi}\norm{\V\trans \B \W_{t})^{-1}}}
{\sigma_k(\Lambda_{\V}) - \norm{\V\trans \B\xi}\norm{(\V\trans \B \W_{t})^{-1}}} \nn\\
&=\frac{\norm{\Lambda_{\Vc}} \tan \theta_t + \frac{\norm{\Vc\trans \B \xi}}{\cos \theta_t}}
{\sigma_k(\Lambda_{\V}) - \frac{\norm{\V\trans \B\xi}}{\cos \theta_t}} \nn\\
&\le \tan \theta_t \frac{|\lambda_{k+1}| + \frac{\norm{\xi}_\B}{\sin\theta_t}  }
{|\lambda_{k}| - \frac{\norm{\xi}_\B}{\cos\theta_t} }
\end{align}
Since $\norm{\xi}_\B \le \frac{|\lambda_k| - |\lambda_{k+1}|}{4}\min\{ \sin\theta_t, \cos\theta_t\}$, we have that:
\begin{align}
\tan \theta_{t+1} &\le 
\frac{|\lambda_{k}| + 3|\lambda_{k+1}|}{3|\lambda_k| + |\lambda_{k+1}|} \tan \theta_t \\
&= (1- \frac{2(|\lambda_k| - |\lambda_{k+1}|)}{3|\lambda_k| + |\lambda_{k+1}| }) \tan\theta_t
\le \exp(-\frac{|\lambda_k| - |\lambda_{k+1}|}{2|\lambda_k|})\tan\theta_t
\end{align}
Recall in this problem $\rho = 1-\frac{\abs{\lambda_{k+1}}}{\abs{\lambda_k}}$, therefore, we know:
\begin{equation}
\sin \theta_t \le \tan \theta_t \le \exp(- \frac{\rho}{2} \cdot t)
\tan \theta_0
\le \exp(- \frac{\rho}{2} \cdot t) \frac{1}{\cos \theta_0}
\end{equation}
If we want $\sin\theta_t \le \epsilon$, which gives iterations:
\begin{equation}
t \ge \frac{2}{\rho}
\log \frac{1}{\epsilon \cos\theta_0}
\end{equation}

Let $f(\W) = \tr(\frac{1}{2}\W\trans\B\W - \W\trans\A\W_t)$. For this problem, we can view $\W$ as a $dk$ dimensional vector, and use linear system to solve this $d$, $k$ dimensional problem.
Therefore, if we represent $\W$ in terms of matrix, the corresponding linear system error is 
$\norm{\W - \B^{-1}\A\W_t}_{\B, F}$, recall $\norm{\W}_{\B, F} = \norm{\B^\frac{1}{2}\W}_F = \sqrt{\tr(\W\trans \B \W)}$. To satisfy the accuracy requirement, we only need
\begin{equation}
\epsilon_{\textrm{des}} = \norm{\xi}^2_{\B, F} \le \frac{(|\lambda_k| - |\lambda_{k+1}|)^2}{16}\min\{\sin^2\theta_t ,   \cos^2\theta_t\}
\end{equation}
Recall we initialize the linear system solver with $\W_t\Gamma_t$ with $\Gamma_t = (\W_t\trans \B \W_t)^{-1} (\W_t \trans \A \W_t)$, we then have
\begin{align}
 \epsilon_{\textrm{init}}& = \norm{\W_t\Gamma_t - \B^{-1}\A \W_t}_{\B, F}^2 = 
 \tr[(\W_t\Gamma_t - \B^{-1}\A \W_t)\trans \B (\W_t\Gamma_t - \B^{-1}\A \W_t)] \nn\\
= &2[f(\W_t\Gamma_t) - f(\B^{-1}\A \W_t)] = 2[\argmin_{\Gamma\in \R^{k\times k}} f(\W_t\Gamma) - f(\B^{-1}\A \W_t)]
\end{align}
Let $\hat{\Gamma}_t = (\V\trans \B \W_t)^{-1} \Lambda_\V (\V\trans \B \W_t)$, and observe
$\norm{\xi}_{\B, F}^2 = \norm{\B^{\frac{1}{2}}\xi}_F^2 = \norm{\V\trans \B\xi}_F^2
+ \norm{\Vc \trans \B\xi}_F^2$ (Pythagorean theorem under $\B$ norm), then we have:
\begin{align}
 \epsilon_{\textrm{init}} = &\norm{\W_t\Gamma_t - \B^{-1}\A \W_t}_{\B, F}^2= 2[\argmin_{\Gamma\in \R^{k\times k}} f(\W_t\Gamma) - f(\B^{-1}\A \W_t)] \nn \\
\le& 2[f(\W_t\hat{\Gamma}_t) - f(\B^{-1}\A \W_t)] 
= \norm{\W_t\hat{\Gamma}_t - \B^{-1}\A \W_t}_{\B, F}^2 \nn \\
= & \norm{\V\trans\B  (\W_t\hat{\Gamma}_t - \B^{-1}\A \W_t)}_F^2 + \norm{\Vc\trans\B  (\W_t\hat{\Gamma}_t - \B^{-1}\A \W_t)}_F^2  \nn\\
= & \norm{\V\trans \B \W_t \hat{\Gamma}_t - \Lambda_\V \V\trans \B \W_t}_F^2
+ \norm{\Vc\trans \B \W_t \hat{\Gamma}_t - \Lambda_{\Vc} \Vc\trans \B \W_t}_F^2
\nn \\
= & 0
+ \norm{\Vc\trans \B \W_t \hat{\Gamma}_t - \Lambda_{\Vc} \Vc\trans \B \W_t}^2_F \nn \\
\le &k\norm{\Vc\trans \B \W_t \hat{\Gamma}_t - \Lambda_{\Vc} \Vc\trans \B \W_t}^2 \nn \\
\le &2k\sin^2\theta_t(\norm{\hat{\Gamma}_t}^2 + \norm{\Lambda_{\Vc}}^2) 
\le 4k|\lambda_1|^2 \tan^2\theta_t
\end{align}

The last step is correct since $\norm{\Lambda_{\Vc}} \le |\lambda_1|$ and
$\norm{\hat{\Gamma}_t} \le \norm{(\V\trans \B \W_t)^{-1}} \norm{\Lambda_\V} \norm{\V\trans \B^\frac{1}{2}}\norm{\B^\frac{1}{2} \W_t} \le \frac{1}{\cos \theta_t}|\lambda_1|$

This means we wish to decrease the ratio of final to initial error smaller than:
\begin{align}
\frac{\epsilon_{\textrm{des}}}{\epsilon_{\textrm{init}}} 
\le & \frac{\rho^2}{64k\gamma^2}\min\{\frac{1}{\cos^2 \theta_t}, \frac{\sin^2 \theta_t}{\cos^4 \theta_t}\} 
\end{align}
where $\gamma = \frac{\abs{\lambda_1}}{\abs{\lambda_k}}$.
Therefore, a two phase analysis of running time depending on $\theta_t$ is large or small similar to top 1 case
would gives the total runtime:
\begin{align*}
&\frac{2}{\rho} \left( \log\frac{1}{\cos \theta_0} \cdot  \Tcal{\frac{\rho^2 \cos^4 \theta_0}{64 k \gamma^2 }} +\log \frac{1}{\epsilon} \cdot \Tcal{\frac{\rho^2}{64 k \gamma^2}} \right) + \frac{2}{\rho} \left(\nnz{\A} k+\nnz{\B}k + dk^2\right) \log\frac{1}{\epsilon \cos \theta_0},
\end{align*}
if we are using the accelerated gradient descent to solve the linear system, we are essentially solve $k$ disjoint optimization problem, with each problem dimension $d$ and condition number $\cn(\B)$. Directly apply Theorem \ref{thm:AGD_guarantee} gives runtime
\begin{align*}
&O\left( \frac{ \nnz{\B}k \sqrt{\cn(\B)}}{\rho} \left( \log\frac{1}{\cos \theta_0} \log \frac{k\gamma}{\rho\cos \theta_0}  +\log \frac{1}{\epsilon} \log \frac{k\gamma}{\rho} \right)  + \frac{\left(\nnz{\A}k+dk^2\right)}{\rho} \log\frac{1}{\epsilon \cos \theta_0} \right).
\end{align*}
\end{proof}

Finally, since both results Theorem \ref{thm:main} and Theorem \ref{thm:main_cca} are stated in terms of initialization $\theta_0$, here we will give probablistic guarantee for random initialization.
\begin{lemma}[Random Initialization]\label{lem:randinit}
Let top $k$ eigen-vector be $\V \in \R^{d\times k}$, and the remaining eigen-vector be $\Vc \in \R^{d\times (d-k)}$. If we initialize $\W_0$ as in Algorithm \ref{algo:appro-ortho}, then
With at least probability $1-\eta$, we have:
\begin{equation}
\tan \theta_0 = \norm{\Vc\trans \B \W_0 (\V\trans \B \W_0)^{-1}} \le O(\frac{\sqrt{\cn(\B)dk}}{\eta})
\end{equation}
\end{lemma}

\begin{proof}
Recall $\tilde{\W}$ is entry-wise sampled from standard Gaussian, and 
\begin{align}
\tan \theta_0 = &\norm{\Vc\trans \B \W_0 (\V\trans \B \W_0)^{-1}} 
= \norm{\Vc\trans \B \tilde{\W}_0 (\V\trans \B \tilde{\W}_0)^{-1}}  \le
\frac{\norm{\Vc\trans \B \tilde{\W}_0 }}{\sigma_k(\V\trans \B \tilde{\W}_0)  } \nn \\
\le &\frac{\norm{\Vc\trans \B\tilde{\Vc}}}{\sigma_k(\V\trans \B\tilde{\V})} \frac{\norm{\tilde{\Vc}\trans \tilde{\W}_0}}{\sigma_k(\tilde{\V}\trans \tilde{\W}_0)}
\end{align}
Where $\tilde{\Vc}, \tilde{\V}$ are the right singular vectors of $\Vc\trans \B, \V\trans \B$ respectively.
Then, we have first term:
\begin{align}
\frac{\norm{\Vc\trans \B\tilde{\Vc}}}{\sigma_k(\V\trans \B\tilde{\V})} = \frac{\norm{\Vc\trans \B}}{\sigma_k(\V\trans \B)} \le \frac{\norm{\Vc\trans \B^{\frac{1}{2}}}\norm{\B^{\frac{1}{2}}}}{\sigma_k(\V\trans \B^\frac{1}{2})\sigma_{\min}(\B^\frac{1}{2})} = \cn(\B)^\frac{1}{2}
\end{align}
The last step is true since both $\Vc\trans \B^\frac{1}{2}$ and $\V\trans \B^\frac{1}{2}$ are orthonormal matrix.

For the second term, we know 
$\norm{\tilde{\Vc}\trans \tilde{\W}_0} \sim O(\sqrt{d} + \sqrt{k})$ with high probability, and by 
equation 3.2 in \cite{rudelson2010non}
we know $\sigma_k(\tilde{\V}\trans \tilde{\W}_0) \ge \frac{\eta}{\sqrt{k}}$ with probability at least $1-\eta$, which finishes the proof.
\end{proof}

\subsection{CCA Setting}
Since our approach to CCA directly calls Algorithm \ref{algo:appro-ortho} for solving generalized eigenvalue problem as subroutine, most of the theoretical property should be clear other than random projection step in Algorithm \ref{algo:appro-cca}. Here, we give following lemma. The proof of Theorem \ref{thm:main_cca} easily follow from the combination of this lemma and Theorem \ref{thm:main_k}.

\begin{lemma}\label{lem:ccatoindividual}
If the $\left(\begin{array}{c c}\bar{\W}_x \\ \bar{\W}_y\end{array}\right)$ as constructed in Algorithm~\ref{algo:appro-cca} has angle at most $\theta$ with the true top-$2k$ generalized eigenspace of $\A,\B$, then with probability $1-\zeta$, both $\W_x$, $\W_y$ has angle at most $O(k^2\theta/\zeta^2)$ with the true top-$k$ canonical space of $\X, \Y$. 
\end{lemma}

\begin{proof}
We will prove this for $\W_y$, the proof for $\W_x$ follows directly from same strategy.

Recall $\B = \left(\begin{array}{cc}
\S_{xx} & 0 \\ 0 & \S_{yy}
\end{array}\right)$. Let $\Phi \in \R^{d_1\times k}$ be the true top $k$ subspace of $\X$ and $\Psi \in \R^{d_2\times k}$ be the true top $k$ subspace of $\Y$.Then by construction we know the top $2k$ subspace should be $\frac{1}{\sqrt{2}}\left(\begin{array}{cc} \Phi & -\Phi \\ \Psi & \Psi\end{array}\right)$.

\newcommand{\RR}{\mathbf{R}}
\newcommand{\EE}{\mathbf{E}}
By properties of principal angle, we know there exists an orthonormal matrix $\RR\in \R^{2k\times 2k}$ such that 
$$
\|\frac{1}{\sqrt{2}}\B^{1/2} \left(\begin{array}{cc} \Phi & -\Phi \\ \Psi & \Psi\end{array}\right)\RR - \B^{1/2} \left(\begin{array}{c} \bar{\W}_x \\ \bar{\W}_y\end{array}\right)\| \le 2\sin \frac{\theta}{2}.
$$

In particular, if we only look at the last $d_2$ rows, we have
$$\|\frac{1}{\sqrt{2}}\S_{yy}^{1/2} \left(\begin{array}{cc} \Psi & \Psi\end{array}\right)\RR - \S_{yy}^{1/2} \bar{\W}_y\| \le 2\sin \frac{\theta}{2}.
$$

Let $\U$ be the random Gaussian projection we used, and let $\RR\U = \left(\begin{array}{c} \U_1 \\ \U_2\end{array}\right)$, we know

\begin{align*}
\S_{yy}^{1/2} \bar{\W}_y \U & =  \frac{1}{\sqrt{2}}\S_{yy}^{1/2} \left(\begin{array}{cc} \Psi & \Psi\end{array}\right)\left(\begin{array}{c} \U_1 \\ \U_2\end{array}\right) + \EE \\ &= \frac{1}{\sqrt{2}}\S_{yy}^{1/2} \Psi(\U_1+\U_2) + \EE,
\end{align*}
where $\EE$ is the error (after multipled by random matrix $\U$), with $\|\EE\|\le O(2\sqrt{k}\sin \frac{\theta}{2}) \le O(\sqrt{k}\theta)$.

Let $\V = (\S_{yy}^{1/2} \bar{\W}_y \U)^\top \S_{yy}^{1/2} \bar{\W}_y \U$, the orthonormalization step gives a matrix $\W_y$ that is equivalent (up to rotation) to $\bar{\W}_y \U \V^{-1/2}$. Our goal is to show $\V^{-1/2} \approx ((\U_1+\U_2)^\top(\U_1+\U_2))^{-1/2}$ so we get roughly $\Psi$. 

Note that $\Psi^\top \S_{yy} \Psi = \I$, therefore $\V = \frac{1}{2}(\U_1+\U_2)^\top(\U_1+\U_2) + \EE'$ where the error $\EE' = (\frac{1}{\sqrt{2}}\S_{yy}^{1/2}\Psi(\U_1+\U_2))^\top \EE + \frac{1}{\sqrt{2}}\EE^\top (\S_{yy}^{1/2}\Psi(\U_1+\U_2)) + \EE^\top \EE)$. We know with high probability $\|\U_1+\U_2\| \le O(\sqrt{k})$, with probability at least $1-\zeta$, $\sigma_{min}(\U_1+\U_2) \ge \Omega(\zeta/\sqrt{k})$.
Therefore we know $\sigma_{min}[(\U_1+\U_2)\trans(\U_1+\U_2)] \ge \Omega(\zeta^2/k)$ and $\|\EE'\|\le O(k\theta)$. By matrix perturbation for inverse we know $\|\V^{-1/2} - \sqrt{2}((\U_1+\U_2)^\top(\U_1+\U_2))^{-1/2}\| \le O(k^2\theta/\zeta^2)$. Since $(\U_1+\U_2) ((\U_1+\U_2)^\top(\U_1+\U_2))^{-1/2} = \RR'$ is an orthonormal matrix, we know there's some orthonormal matrix $\RR''$ so that:
\begin{align*}
&\norm{\S_{yy}^{1/2}\W_y - \S_{yy}^{1/2} \Psi \RR''} = 
\norm{\S_{yy}^{1/2} \bar{\W}_y\U\V^{-1/2} - \S_{yy}^{1/2} \Psi \RR'} \\
\le &\norm{\S_{yy}^{1/2} \bar{\W}_y\U\V^{-1/2} - \sqrt{2}\S_{yy}^{1/2} \bar{\W}_y\U
((\U_1+\U_2)^\top(\U_1+\U_2))^{-1/2}} \\
&+ \norm{\sqrt{2}\S_{yy}^{1/2} \bar{\W}_y\U
((\U_1+\U_2)^\top(\U_1+\U_2))^{-1/2} - \S_{yy}^{1/2} \Psi \RR'} 
\le  O(k^2\theta/\zeta^2)
\end{align*}


Therefore the angle between the $\W_y$ and the truth $\Psi$ is bounded by $O(k^2\theta/\zeta^2)$.

\end{proof}

\end{document}